\documentclass{article}

\usepackage[nonatbib,main,final]{neurips_2026}
\usepackage[sort&compress,numbers]{natbib}

\usepackage[utf8]{inputenc} 
\usepackage[T1]{fontenc}    
\usepackage{hyperref}       
\usepackage{url}            
\usepackage{booktabs}       
\usepackage{amsfonts}       
\usepackage{nicefrac}       
\usepackage{microtype}      
\usepackage{xcolor}         

\usepackage[utf8]{inputenc} 
\usepackage[T1]{fontenc}    
\usepackage{hyperref}       
\usepackage{url}            
\usepackage{booktabs}       
\usepackage{amsfonts}       
\usepackage{nicefrac}       
\usepackage{microtype}      
\usepackage{xcolor}         

\usepackage{amsmath,amsfonts,bm}

\def\eqref#1{equation~\ref{#1}}

\def\1{\bm{1}}

\DeclareMathAlphabet{\mathsfit}{\encodingdefault}{\sfdefault}{m}{sl}
\SetMathAlphabet{\mathsfit}{bold}{\encodingdefault}{\sfdefault}{bx}{n}

\usepackage{graphicx}
\usepackage{float}
\usepackage[sort&compress,numbers]{natbib}
\usepackage{wrapfig}      
\usepackage{lipsum}       
\usepackage{caption}      
\makeatletter\renewcommand\NAT@spacechar{~}\makeatother 
\usepackage{xurl}
\usepackage{amsmath}
\usepackage{amssymb}
\usepackage{multirow}
\usepackage{longtable}
\usepackage{arydshln}
\usepackage{enumitem}
\usepackage{xcolor} 
\usepackage{tikz}
\usepackage{amsmath}
\usepackage{amssymb}
\usetikzlibrary{shapes.geometric, calc, positioning, arrows.meta, patterns}
\usepackage{svg}

\usepackage{tikz}
\usepackage{amsmath}
\usepackage{amssymb}
\usetikzlibrary{shapes.geometric, calc, positioning, arrows.meta, patterns}
\usepackage{svg}
\usepackage{subcaption}
\usepackage{dirtytalk}
\usepackage{amsthm}
\makeatletter
\def\th@plain{%
  \thm@notefont{}
  \itshape 
}
\theoremstyle{plain}
\usepackage[normalem]{ulem}
\useunder{\uline}{\ul}{}
\usepackage[ruled,vlined,noend]{algorithm2e}

\newtheorem{lemma}{Lemma}

\newtheorem{proposition}{Proposition}
\newtheorem{corollary}{Corollary}
\title{Literati: Towards Anytime Optimal Shape Generalized Trees via AO*}

\author{%
  Nakul Upadhya, Eldan Cohen \\
  Department of Mechanical and Industrial Engineering\\
  University of Toronto, Toronto, Canada \\
  \texttt{nakul.upadhya@mail.utoronto.ca, eldan.cohen@utoronto.ca} 
  }

\SetCommentSty{mycommfont}

\begin{document}

\maketitle

\begin{abstract}
Decision trees are prized for their interpretability and strong performance on tabular data, but popular greedy top-down induction algorithms can yield suboptimal and unnecessarily complex structures. Optimal decision tree methods address this through global optimization, yet remain restricted to axis-aligned threshold splits, which limit the expressivity of each node and often force deep, complex trees to capture non-linear feature effects. Shape Generalized Trees (SGTs) generalize threshold splits to learnable univariate shape functions, improving expressivity and enabling more compact trees. However, existing SGT induction algorithms are greedy and offer no optimality guarantees. In this work, we introduce Literati, the first algorithm for optimal SGT induction. We propose a novel AND/OR graph formulation of the problem that jointly optimizes tree structure and shape function complexity. To solve this AND/OR graph, we develop an AO*-based algorithm with two enhancements that improve anytime performance while preserving optimality: a secondary heuristic for OR-node selection and a round-robin policy for AND-node exploration. Across 24 real-world datasets, Literati achieves higher training and test accuracy than state-of-the-art tree approaches.
\end{abstract}

\section{Introduction}
Decision trees are one of the most widely used approaches for tabular datasets~\cite{grinsztajn2022tree}, largely due to their inherent interpretability~\cite{luvstrek2016makes} and robust performance~\cite{grinsztajn2022tree}. Their hierarchical structure of simple threshold comparisons ($x_j \leq \theta$) can be easily visualized and audited, making them the model of choice when transparency is paramount. However, the common greedy top-down induction algorithms, such as CART~\cite{CART}, suffer from a fundamental limitation: by making locally optimal splits at each node, these methods often produce suboptimal and unnecessarily complex trees~\cite{gosdt,chaouki2025branches,kohler2025breiman,demirovic2023blossom}. 

Recent years have witnessed significant progress in decision tree algorithms that overcome greedy suboptimality through global optimization~\cite{gosdt,dl85,van2023necessary,kohler2025breiman,chaouki2025branches}. These methods frame tree induction as an optimization problem over a regularized objective that balances accuracy and tree complexity, effectively optimizing both predictive performance and sparsity simultaneously~\cite{gosdt,dl85,kohler2025breiman}. Such non-greedy approaches have yielded trees that empirically generalize better to unseen data while reducing model complexity~\cite{gosdt,dl85}.

While these algorithmic advances were underway, other work has focused on reimagining how tree branching is performed. Specifically, \citet{upadhyaempowering} proposed the Shape Generalized Tree (SGT), which replaces the axis-aligned threshold comparison with an axis-aligned \emph{Shape Function}. By enabling non-linear decision boundaries at each node, SGTs have been empirically shown to achieve superior predictive performance with a more compact tree structure compared to threshold trees~\citep{upadhyaempowering}. However, existing methods for learning SGTs rely on greedy induction, resulting in suboptimal trees and limited control over the learned shape functions. 

In this work, we introduce \texttt{Literati}\footnote{Named after the bonsai styling technique characterized by dramatic trunk shapes and minimal branches.}, a novel approach to Optimal SGT induction. Our approach formulates SGT induction within an AND/OR graph search framework and leverages an AO*-inspired algorithm to efficiently search for an optimal SGT. Our contributions are as follows: \emph{1) We propose a novel AND/OR graph formulation of the sparse SGT induction problem that jointly optimizes tree structure and shape function complexity, and prove that the resulting formulation yields optimal SGTs. 2) We introduce Literati, an AO*-based algorithm augmented with an inadmissible secondary heuristic for OR-node selection and round-robin AND-node exploration, yielding strong anytime performance while preserving optimality guarantees. 3) We evaluate Literati on a range of real-world datasets and demonstrate that our approach consistently outperforms state-of-the-art tree induction baselines in both train and test accuracy.} 

\section{Background}

\subsection{AND/OR Graph Search}

An AND/OR graph $\mathcal{G}$ encodes problems whose solutions decompose into independent subproblems through a state-action model $(\mathcal{S}, \ell, \mathcal{A}, F, c)$~\cite{nilsson2014principles,bonet2005algorithm}. Here, $\mathcal{S}$ is a set of states, $\ell \subseteq \mathcal{S}$ is the set of terminal states, $\mathcal{A}(s)$ is the set of actions available at a non-terminal state $s \in \mathcal{S} \setminus \ell$, $F : \mathcal{S} \times \mathcal{A} \to 2^\mathcal{S}$ maps each state-action pair to its non-empty set of successor states, and $c : \mathcal{S} \times \mathcal{A} \to \mathbb{R}_{\geq 0}$ assigns a non-negative cost to each action. Each action $a \in \mathcal{A}(s)$ has AND dynamics over its successors: $a$ solves $s$ only if every $s' \in F(s, a)$ is itself solved. Each non-terminal state $s$ has OR dynamics over its actions: $s$ is solved as soon as any single $a \in \mathcal{A}(s)$ solves it. For example, a state $s$ with two actions $a_1, a_2$ such that $F(s, a_1) = \{s_1, s_2\}$ and $F(s, a_2) = \{s_3, s_4\}$ is solved when $\{s_1 \text{ AND } s_2\}$ are solved OR $\{s_3 \text{ AND } s_4\}$ are solved. Terminal states are solved by definition. 

\begin{wrapfigure}[18]{r}{0.5\textwidth}
\begin{algorithm}[H]
\microtypesetup{protrusion=false}
\DontPrintSemicolon
\SetKwInOut{Input}{Input}
\SetKwInOut{Output}{Output}
\Input{$V$, $\mathcal{G}$, $s_0$}
\Output{$\pi$}
$G \gets \{s_0\}$; \quad $\pi \gets \emptyset$\;
\While{$s_0$ is not \textup{COMPLETE}}{
    \tcp{Selection}
  $s \gets$ follow $\pi$ from $s_0$ to a non-terminal tip\;
  \tcp{Expansion}
  $G \gets G \cup \text{Successors}(s)$\; 
  Initialize $V(s')$ for each $s' \in \textup{Successors}(s)$\; 
  \tcp{Backpropagation}
  Update $V$, $\pi$, and COMPLETE status along $s \rightsquigarrow s_0$ 
}
\Return{$\textup{CollectSolutionGraph}(\pi, G)$}
\caption{AO$^{*}$ search.}
\label{alg:ao_main_loop}
\end{algorithm}
\end{wrapfigure}
Problems solvable by dynamic programming admit a natural formulation as AND/OR graph search, with subproblems mapped to states and their decompositions to actions~\citep{martelli1973additive,martelli1975dynamic}. While bottom-up dynamic programming must evaluate all subproblems in the graph, a best-first search guided by admissible heuristic estimates can find a provably optimal solution while expanding only a relevant subset of the search space~\citep{martelli1973additive,originalAOstar,chaouki2025branches}. AO* realizes this approach for acyclic AND/OR graphs~\cite{nilsson2014principles,bonet2005algorithm}. It maintains a partial graph $G \subseteq \mathcal{G}$, initialized to the root node $\{s_0\}$, together with an admissible heuristic $V$ that lower-bounds the optimal cost at every state~\cite{bonet2005algorithm}, and a \emph{best partial policy} $\pi$ over $G$. A partial policy assigns an action $\pi(s) \in \mathcal{A}(s)$ to each non-terminal state in $G$ and is greedy with respect to $V$ at every such state, i.e.\ $\pi(s) = \arg \min_{a \in \mathcal{A}(s)}[ c(s,a) + \sum_{s' \in F(s,a)} V(s')]$. Both $G$ and $\pi$ are refined iteratively through three phases: \textbf{selection} follows $\pi$ from $s_0$ until reaching a non-terminal tip state $s \in G$; \textbf{expansion} adds the successors of $s$ under each $a \in \mathcal{A}(s)$ to $G$ and initializes their heuristic values; and \textbf{backpropagation} revises $V$, and consequently $\pi$, bottom-up through $s$ and its ancestors. A state is marked COMPLETE once all its successors under $\pi$ are COMPLETE, with terminal states COMPLETE by definition, and the algorithm terminates when $s_0$ is COMPLETE, at which point $\pi$ is optimal~\citep{bonet2012action,nilsson2014principles,bonet2005algorithm}. The optimal solution graph $G^* \subseteq G$ is then collected by traversing $G$ from $s_0$ along $\pi$ until reaching terminal states. Pseudocode is provided in Algorithm~\ref{alg:ao_main_loop}. 

\subsection{Optimal Decision Trees}
While computationally efficient, greedy tree induction methods, such as CART~\cite{CART} and C4.5~\cite{quinlan2014c4}, often produce large trees with poor performance~\cite{gosdt,kohler2025breiman}. To address this drawback, there exists a growing body of work on globally optimal tree induction. Although the problem is NP-complete~\cite{laurent1976constructing}, recent algorithmic advances have rendered it tractable for limited-depth trees on moderate-sized datasets. A prominent family employs branch-and-bound search with caching~\cite{dl85,OSDT,demirovic2022murtree}, with GOSDT~\cite{gosdt} and STreeD~\cite{van2023necessary} extending these methods to broader classes of objectives. All of these approaches require global pre-binarization of features. DPDT~\cite{kohler2025breiman} instead uses a CART-based discretization that adaptively selects candidate thresholds at each node, while QuantBnB~\cite{pmlr-v162-mazumder22a} and ConTree~\cite{brita2025optimal} operate directly on continuous features. There is also a growing line of work on \emph{anytime optimal} tree methods that seek to identify high-quality solutions early in the search process~\cite{demirovic2023blossom,kiossou2022time,kiossou2026generic}. Additionally, recent approaches, such as SPLIT and LicketySPLIT, trade optimality guarantees for computational efficiency by replacing parts of the search using greedy induction \cite{babbar2025near}. 

Particularly relevant to our work are AND/OR graph-based methods. \citet{originalAOstar} employed AO* to seek optimal DTs in a cost-sensitivity context, while  Branches~\cite{chaouki2025branches} recently applied AO* to the sparse tree induction problem. While Literati also adopts the underlying AO* mechanism, it departs from prior work in two fundamental ways: (1) we propose a novel AND/OR graph formulation that enables optimal induction of SGTs (Section~\ref{sec:aograph}), and (2) we propose key enhancements to the AO* selection procedure that substantially improve anytime performance while preserving optimality (Section~\ref{sec:anytimeaostar}).

\subsection{Shape Generalized Trees}
Shape Generalized Trees (SGTs)~\cite{upadhyaempowering} generalize decision trees by replacing threshold splits with learnable univariate shape functions $f(x_j)$ that map feature values to left or right branches. By capturing non-linear relations between features and the target within a single node, SGTs can reduce repeated branching on the same feature, resulting in more compact trees. As each split remains univariate, the learned shape functions can be easily visualized, preserving the inherent interpretability of decision trees~\cite{upadhyaempowering}. 

However, due to their increased expressivity, inducing SGTs remains challenging. Current methods, primarily ShapeCART~\cite{upadhyaempowering}, are greedy and select shape functions by locally optimizing an impurity measure (e.g., Gini or entropy), yielding globally suboptimal trees. Post-hoc refinement methods like TAO~\cite{TAO} can improve these trees but are limited by the quality of the greedy initialization.

\section{Literati}\label{sec:literati}
In this work, we consider a supervised classification setting over a dataset $\mathcal{D} = \{(\mathbf{x}_n, y_n)\}_{n=1}^N$, where each $\mathbf{x}_n \in \mathbb{R}^M$ is an $M$-dimensional feature vector and each $y_n \in \mathcal{C}$ is the corresponding class label. We seek a Shape Generalized Tree (SGT) that minimizes classification error while controlling model complexity along two axes: the number of internal nodes, which governs overall model size, and the number of discontinuities in the learned shape functions, which governs the complexity of individual branching rules. Specifically, we solve the following regularized optimization problem:
\begin{equation} \label{eq:objective}
\begin{aligned}
    \min_{T \in \mathcal{T}}&\quad \text{Err}(T, \mathcal{D}) +  \sum_{v \in \mathcal{I}(T)} \left[\lambda + \alpha \mathbf{K}(f_v) \right]\\
    \text{s.t.}&\quad \mathbf{K}(f_v) \leq K \quad \forall\, v \in \mathcal{I}(T);\quad \mathbf{D}(T) \leq D
\end{aligned}
\end{equation}
where $\mathcal{T}$ denotes the set of all binary, axis-aligned SGTs, $\text{Err}(T, \mathcal{D})$ is the misclassification count of tree $T$ on dataset $\mathcal{D}$, and $\mathcal{I}(T)$ is the set of internal nodes of $T$. The hyperparameters $\lambda \geq 0$ and $\alpha \geq 0$ control the trade-off between error and model complexity.  At each internal node $v$, a single feature $j_v$ is selected and a piecewise constant shape function $f_v \colon \mathbb{R} \to \{0,1\}$ maps the value of feature $j_v$ to the left or right child. We define the complexity of $f_v$ as the number of discontinuities (change-points) in this function: $\mathbf{K}(f_v) = \sum_{i=1}^{u_v - 1} \mathbf{1}\!\bigl[f_{v}\!\bigl(x_{j_v}^{(i)}\bigr) \neq f_{v}\!\bigl(x_{j_v}^{(i+1)}\bigr)\bigr]$ where $x_{j_v}^{(1)} \leq \cdots \leq x_{j_v}^{(u_v)}$ are the sorted unique values of feature $j_v$ among instances reaching node $v$. Note that $\mathbf{K}(f_v) = 1$ recovers a standard threshold split, with higher values corresponding to more complex shape functions. The constraint $\mathbf{K}(f_v) \leq K$ with $K \geq 1$ bounds this complexity to ensure that individual branching rules remain amenable to human inspection. Similarly, the constraint on the maximum depth of the tree $\mathbf{D}(T) \leq D,\ D \geq 1$ limits the size of the tree and consequently improves interpretability~\cite{luvstrek2016makes}. Both constraints also help mitigate overfitting and can be tuned as hyperparameters.

Equation \ref{eq:objective} is recursively decomposable across subtrees, and at internal node $v$, the optimal subtree cost is obtained by choosing the better of becoming a leaf or splitting with the best shape function:
\begin{gather}
    \mathcal{L}_d^*(\mathcal{D}_v) = \min{\left(\text{Err}(\mathcal{D}_v),\; \min_{f_v :\, \mathbf{K}(f_v) \leq K} \lambda + \alpha\, \mathbf{K}(f_v) + \mathcal{L}_{d-1}^*(\mathcal{D}_v(f_v)_L) + \mathcal{L}_{d-1}^*(\mathcal{D}_v(f_v)_R)\right)}, \label{eq:node_problem}
\end{gather}
where $\mathcal{L}_d^*(\mathcal{D})$ denotes the cost of an optimal subtree for data $\mathcal{D}$ with depth budget $d$, computed recursively. The first argument, $\text{Err}(\mathcal{D}_v)$, is the cost of assigning the majority class at node $v$ without further splitting. $\mathcal{D}_v(f_v)_L$, $\mathcal{D}_v(f_v)_R$ are the left and right data partitions obtained by applying $f_v$. This recursion terminates when we reach the maximum depth, in which case $\mathcal{L}_{0}^*(\mathcal{D}_v)=\text{Err}(\mathcal{D}_v)$, or the node is pure with no error, i.e.\ $\text{Err}(\mathcal{D}_v) = 0$.

\subsection{AND/OR Graph Representation}\label{sec:aograph}
\definecolor{commitColor}{HTML}{913F5A}
\definecolor{splitColor}{HTML}{014C82}
\definecolor{leafColor}{HTML}{697A21}
\definecolor{refineColor}{HTML}{C06F0C}
\begin{figure}[!t]
\centering
\resizebox{\textwidth}{!}{%
\begin{tikzpicture}[
    scale=1,
    transform shape,
    font=\huge,
    datastate/.style={
        circle,
        draw=black,
        line width=2.5pt,
        fill=white,
        minimum size=4.5cm,
        align=center,
        font=\bfseries\Huge
    },
    shapestate/.style={
        rectangle,
        draw=black,
        line width=2pt,
        fill=white,
        minimum height=2.5cm,
        minimum width=6.5cm,
        inner sep=0pt
    },
    leaficon/.style={
        regular polygon,
        regular polygon sides=8,
        draw=leafColor,
        line width=2.5pt,
        fill=white,
        minimum size=2.5cm
    },
    cost/.style={
        font=\bfseries\Huge,
        fill=white,
        inner sep=4pt
    },
    featlabel/.style={
        font=\bfseries\Huge,
        text=black,
        anchor=south,
        inner sep=3pt
    },
    tick/.style={
        draw=black,
        line width=1.5pt,
        cap=round
    },
    shadedfill/.style={
        pattern=north east lines, 
        pattern color=black
    }
]

\def\lvloneY{-7cm}
\def\lvltwoY{-16cm}
\def\hgap{8.5cm}


\node (root) [datastate] {$(\mathcal{D}, d)$};

\node (leaf) [leaficon, right=of root, xshift=5cm] {};
\draw [->, -{Stealth[scale=1.5]}, line width=2.5pt, color=leafColor] (root.east) -- (leaf.west) node[midway, above, font=\bfseries\Huge, text=leafColor] {Leaf};


\node (s1) [shapestate] at (-1.5*\hgap, \lvloneY) {};
\fill[shadedfill] ($(s1.south west)!0.10!(s1.south east)$) rectangle (s1.north east);
\foreach \t in {0.10, 0.40, 0.85} \draw[tick] ($(s1.south west)!\t!(s1.south east)$) -- ++(0, 0.5);
\node (l1) [featlabel] at (s1.north) {$x_1: \mathbf{k}=\{1\}$};

\draw [->, -{Stealth[scale=1.5]}, line width=2.5pt, color=splitColor] (root) -- (l1.north east) node[midway, cost, text=splitColor] {$\lambda + \alpha$};
\node [font=\bfseries\Huge, text=splitColor] at ($(root)!0.5!(s1) + (-3.0cm, 1.5cm)$) {Split};

\node (s2) [shapestate] at (-0.5*\hgap, \lvloneY) {};
\fill[shadedfill] ($(s2.south west)!0.40!(s2.south east)$) rectangle (s2.north east);
\foreach \t in {0.10, 0.40, 0.85} \draw[tick] ($(s2.south west)!\t!(s2.south east)$) -- ++(0, 0.5);
\node (l2) [featlabel] at (s2.north) {$x_1: \mathbf{k}=\{2\}$};
\draw [->, -{Stealth[scale=1.5]}, line width=2.5pt, color=splitColor] (root) -- (l2.north east) node[midway, cost, text=splitColor] {$\lambda + \alpha$};

\node (s3) [shapestate] at (0.5*\hgap, \lvloneY) {};
\fill[shadedfill] ($(s3.south west)!0.85!(s3.south east)$) rectangle (s3.north east);
\foreach \t in {0.10, 0.40, 0.85} \draw[tick] ($(s3.south west)!\t!(s3.south east)$) -- ++(0, 0.5);
\node (l3) [featlabel] at (s3.north) {$x_1: \mathbf{k}=\{3\}$};
\draw [->, -{Stealth[scale=1.5]}, line width=2.5pt, color=splitColor] (root) -- (l3.north west) node[midway, cost, text=splitColor] {$\lambda + \alpha$};

\node (s4) [shapestate] at (1.5*\hgap, \lvloneY) {};
\fill[shadedfill] ($(s4.south west)!0.11!(s4.south east)$) rectangle (s4.north east);
\foreach \t in {0.11, 0.47, 0.65, 0.90} \draw[tick] ($(s4.south west)!\t!(s4.south east)$) -- ++(0, 0.5);
\node (l4) [featlabel] at (s4.north) {$x_2: \mathbf{k}=\{1\}$};
\draw [->, -{Stealth[scale=1.5]}, line width=2.5pt, color=splitColor] (root) -- (l4.north west) node[midway, cost, text=splitColor] {$\lambda + \alpha$};
\node [right=of s4, font=\Huge] {\dots};


\def\commitOffset{3cm} 
\node (dL) [datastate] at ($(s1) + (-\commitOffset, \lvltwoY - \lvloneY)$) {$\mathcal{D}_s(f_1^{\mathbf{k}})_L,$\\$d-1$};
\node (dR) [datastate] at ($(s1) + (\commitOffset, \lvltwoY - \lvloneY)$) {$\mathcal{D}_s(f_1^{\mathbf{k}})_R,$\\$d-1$};

\draw [->, -{Stealth[scale=1.5]}, line width=2.5pt, color=commitColor] (s1.south) -- (dL.north);
\draw [->, -{Stealth[scale=1.5]}, line width=2.5pt, color=commitColor] (s1.south) -- (dR.north);
\draw [line width=2pt, color=commitColor] ($(s1.south)!0.35!(dL.north)$) to[bend right=45] ($(s1.south)!0.35!(dR.north)$);
\node [font=\bfseries\Huge, align=right, fill=white, text=commitColor, anchor=east] at ($(s1.south) + (-1.5cm, -1.5cm)$) {Commit};

\node (r1) [shapestate] at ($(s1) + (1.2*\hgap, \lvltwoY - \lvloneY)$) {}; 
\fill[shadedfill] ($(r1.south west)!0.10!(r1.south east)$) rectangle ($(r1.south west)!0.40!(r1.south east) + (0, 2.45cm)$); 
\foreach \t in {0.10, 0.40, 0.85} \draw[tick] ($(r1.south west)!\t!(r1.south east)$) -- ++(0, 0.5);
\node (lr1) [featlabel] at (r1.north) {$x_1: \mathbf{k}=\{1, 2\}$};

\node (r2) [shapestate] at ($(r1) + (\hgap, 0)$) {};
\fill[shadedfill] ($(r2.south west)!0.10!(r2.south east)$) rectangle ($(r2.south west)!0.85!(r2.south east) + (0, 2.45cm)$);
\foreach \t in {0.10, 0.40, 0.85} \draw[tick] ($(r2.south west)!\t!(r2.south east)$) -- ++(0, 0.5);
\node (lr2) [featlabel] at (r2.north) {$x_1: \mathbf{k}=\{1, 3\}$};

\draw [->, -{Stealth[scale=1.5]}, line width=2.5pt, color=refineColor] (s1.south) -- (lr1.north west) node[midway, cost, text=refineColor] {$\alpha$};
\draw [->, -{Stealth[scale=1.5]}, line width=2.5pt, color=refineColor] (s1.south) -- (lr2.north west) node[midway, cost, text=refineColor] {$\alpha$};
\node [font=\bfseries\Huge, fill=white, text=refineColor] at ($(s1.south) + (13cm, -2.5cm)$) {Refinement};

\node (legend_origin) at ($(s4.south)!0.5!(r2.east) + (1cm, -3cm)$) {}; 

\node (lbox_w) [rectangle, draw=black, thin, fill=white, minimum width=1.5cm, minimum height=0.8cm] at (legend_origin) {};
\node [right=of lbox_w, anchor=west, font=\Huge] {Left partition};

\node (lbox_s) [rectangle, shadedfill, minimum width=1.5cm, minimum height=0.8cm, below=of lbox_w, yshift=0.2cm] {};
\node [right=of lbox_s, anchor=west, font=\Huge] {Right partition};

\end{tikzpicture}
}
\caption{Visual representation of the AND/OR graph for Literati following the hypergraph notation from \citet{nilsson2014principles}. Circle nodes correspond to data states and rectangle nodes correspond to Shape Function States. Each shape function state contains the candidate shape function being built. Edges connected by arcs are related via AND dynamics, edges without arcs are related via OR dynamics }
\label{fig:aograph}
\end{figure}
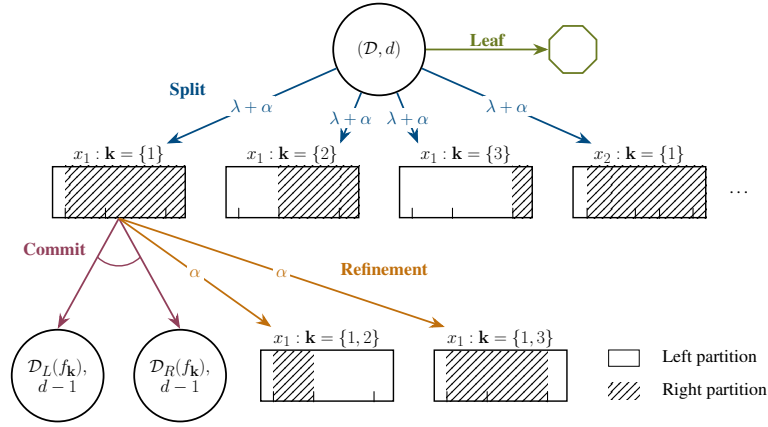

As Equation~\ref{eq:node_problem} decomposes separably over left and right subtrees, it can be solved via dynamic programming~\cite{van2023necessary} and consequently admits a natural formulation as AND/OR graph search~\citep{martelli1973additive,martelli1975dynamic,originalAOstar}. 
A direct translation of Equation~\ref{eq:node_problem} into this framework would require enumerating every feasible shape function at each internal node, leading to an explosion in the size of the partial search graph $G$. Instead, we further decompose Equation~\ref{eq:node_problem} by incrementally constructing shape functions one change-point at a time.
This finer decomposition allows AO* to recognize suboptimal higher-order partial shape functions before they are admitted into $G$. More formally, let $\mathcal{U}^{(j)}_v$ be the set of unique values of feature $j$ from samples that reach internal node $v \in T$. The objective value at a node is defined as: 
\begin{gather}
    \mathcal{L}_d^*(\mathcal{D}_v) = \min {\left(\text{Err}(\mathcal{D}_v),\; \min_{j \in \{1, \ldots, M\}} \; \min_{u\in \mathcal{U}^{(j)}_v}  \; \lambda + \alpha + \mathcal{L}_d^*(\mathcal{D}_v,  j, \{u\})\right)}, \label{eq:node_full} \\
    \resizebox{0.9\columnwidth}{!}{$\displaystyle
    \mathcal{L}_d^*(\mathcal{D}_v, j, \mathbf{k}) = \min \Bigg( \mathcal{L}_{d-1}^*(\mathcal{D}_v(f_j^{\mathbf{k}})_L) + \mathcal{L}_{d-1}^*(\mathcal{D}_v(f_j^{\mathbf{k}})_R),\  \alpha + \min_{\substack{u\in \mathcal{U}^{(j)}_v,\, \max(\mathbf{k}) < u,\\ |\mathbf{k} \cup \{u\}| \leq K}} \mathcal{L}_d^*(\mathcal{D}_v, j, \mathbf{k} \cup \{u\})\Bigg).\label{eqn:shaperecurrence}
$}
\end{gather}
Here, $f_j^{\mathbf{k}}$ denotes the piecewise constant shape function on $j$ defined by the change-point set $\mathbf{k}$. At each stage of the recurrence (Equation~\ref{eqn:shaperecurrence}), we can either recurse into the resulting left and right subtrees, or introduce an additional change-point to the shape function at a cost of $\alpha$. For a state $s$ corresponding to a tree vertex $v$, we write $\mathcal{D}_s = \mathcal{D}_v$. Given this decomposition, we propose the following state-action model (depicted in Figure~\ref{fig:aograph}):

\paragraph{Data States.} A data state $s= (\mathcal{D}_s, d)$, where $\mathcal{D}_s \subseteq \mathcal{D}$ is the subset of training instances reaching this node and $d \in \{0, 1, \ldots, D\}$ is the remaining depth budget. The root of the AND/OR graph is $s_0 = (\mathcal{D}, D)$. All children of a data state are related via OR dynamics. The available actions at a non-terminal data state are: 
\begin{enumerate}[noitemsep,leftmargin=1.5em,labelsep=0.25em]
    \item[(i)] \textbf{Split}: the split actions $a_{j,u}\ \forall j\in \{1, \ldots, M\},\ \forall u \in \mathcal{U}_s^{(j)}$. Each action transitions to the shape function state corresponding to the shape function on feature $j$ with an initial change-point on $u$. More formally, $F(s, a_{j,u}) = (\mathcal{D}_s, d, j, \{u\})$ with $c(s, a_{j,u}) = \lambda + \alpha$.
    \item[(ii)] \textbf{Leaf}: The terminal action $\bar{a}$ transitions $s$ to its terminal state $\overline{(\mathcal{D}_s, d)} \in \ell$, with $c(s, \bar{a}) = 0$.
\end{enumerate}

\paragraph{Terminal States.} A terminal state $\overline{(\mathcal{D}_s, d)} \in \ell$ is a data state at which no further actions are available and which corresponds to a leaf node in the final tree.  Data states are terminal by default when the depth budget is exhausted ($d = 0$) or all instances in $\mathcal{D}_s$ share the same class label.

\paragraph{Shape Function States.} A shape function state $s_{\text{shape}} = (\mathcal{D}_s, d, j, \mathbf{k})$ represents a partially constructed branching rule on feature $j \in \{1, \ldots, M\}$, where $\mathcal{D}_s$ and $d$ are inherited from the parent data state. The set $\mathbf{k}$ contains indices of the ordered values in $\mathcal{U}^{(j)}_s$. These indices define the current piecewise constant shape function $f_j^{\mathbf{k}}$, such that $|\mathbf{k}| = \mathbf{K}(f_j^{\mathbf{k}}) \leq K$. The available actions are:
\begin{enumerate}[noitemsep,leftmargin=1.5em,labelsep=0.25em]
    \item[(i)] \textbf{Commit}: apply the current shape function $f_j^{\mathbf{k}}$ to partition $\mathcal{D}_s$ into $\mathcal{D}_s(f_j^{\mathbf{k}})_L$ and $\mathcal{D}_s(f_j^{\mathbf{k}})_R$. The transition function returns a pair of successor data or terminal states, $F(s, a_{\text{commit}}) = \{(\mathcal{D}_s(f_j^{\mathbf{k}})_L,d-1),\allowbreak\, (\mathcal{D}_s(f_j^{\mathbf{k}})_R,d-1)\}$, which are related via AND dynamics: both must be solved, and their costs sum. The cost of this action is $c(s, a_{\text{commit}}) = 0$, as all regularization costs have already been incurred.
    \item[(ii)] \textbf{Refine}: $a_{j,\mathbf{k
    },u}$ introduces an additional change-point $u \in \mathcal{U}^{(j)}_s$ satisfying $u > \max(\mathbf{k})$ and $|\mathbf{k} \cup \{u\}| \leq K$, transitioning to a new shape function state $F(s, a_{j,\mathbf{k},u}) = \{(\mathcal{D}_s, d, j, \mathbf{k} \cup \{u\})\}$ with $c(s, a_{j,\mathbf{k},u}) = \alpha$. The ordering constraint $u > \max(\mathbf{k})$ enforces a canonical construction of shape functions, avoiding redundant exploration of symmetric change-point sets.
\end{enumerate}

The optimal cost function $V^*$ over this AND/OR graph is characterized by the Bellman equations:
\begin{gather}
\resizebox{0.9\columnwidth}{!}{$\displaystyle
    V^*(s) = \begin{cases} \text{Err}(\mathcal{D}_s) & \text{if } s = \overline{(\mathcal{D}_s, d)} \in \ell \\ \min_{a \in \mathcal{A}(s)} Q^*(s, a) & \text{else} \end{cases}, \qquad Q^*(s, a) = c(s, a) + \sum_{s' \in F(s, a)} V^*(s'). \label{eq:bellman}
$}
\end{gather}
Our goal is to identify the optimal policy $\pi^*$ that prescribes an action at every non-terminal state reachable from the root $s_0 = (\mathcal{D}, D)$ under $\pi^*$ itself. The optimal policy is greedy with respect to $V^*$, i.e.\ $\pi^*(s) = \arg \min_{a \in \mathcal{A}(s)} Q^*(s, a)$, and we can use $\pi^*$ to traverse the AND/OR graph and obtain the optimal tree $T^*$.
\begin{proposition}[Validity of AND/OR Graph]\label{prop:validity} The tree $T^*$ derived from the optimal policy $\pi^*$ is an optimal solution to the sparse SGT induction problem posed in Equation \ref{eq:objective}. Proof in Appendix~\ref{sec:ao_validity_proof}.
\end{proposition}

\subsubsection{Admissible Heuristic}
To solve this AND/OR graph with AO*, we maintain a heuristic estimate $V(s)$ for each non-terminal state $s \in \mathcal{S} \setminus \ell$ which induces the current best partial policy $\pi$ that determines the next state to expand. For data states, we initialize $V(\mathcal{D}_s, d)=\min(\text{Err}(\mathcal{D}_s), \lambda +\alpha)$~\cite{chaouki2025branches}.  For shape function states, we initialize $V(\mathcal{D}_s, d, j, \mathbf{k}) = \min\bigl(\alpha, \min(\text{Err}(\mathcal{D}_s(f_j^{\mathbf{k}})_L), \lambda + \alpha) + \min(\text{Err}(\mathcal{D}_s(f_j^{\mathbf{k}})_R), \lambda + \alpha)\bigr)$.

Whenever AO* expands a state $s$, its children are added to the search graph with their initial heuristic values, and $s$ and its ancestors are updated bottom-up as follows:
\begin{gather}\label{eq:estimate_update}
    V(\widetilde{s}) = \min_{a \in \mathcal{A}(\widetilde{s})} Q(\widetilde{s},a), \quad Q(\widetilde{s},a) =  c(\widetilde{s}, a) + \sum_{s' \in F(\widetilde{s}, a)} V(s'),  \quad \forall\, \widetilde{s} \in \{s\} \cup \text{Ancestors}(s).
\end{gather}

\begin{proposition}[Admissibility]\label{prop:admissibility}
The heuristic $V$ satisfies $V(s) \leq V^*(s)$ for every state $s$ throughout the execution of AO*. The proof is provided in Appendix~\ref{sec:admissibility}.
\end{proposition}
By Proposition~\ref{prop:admissibility} and Proposition~\ref{prop:validity}, the tree induced by $\pi^* = \pi$ at termination of AO* (Algorithm~\ref{alg:ao_main_loop}) is globally optimal. Furthermore, the heuristic enables early detection of terminal data states: if $\text{Err}(\mathcal{D}_s) \leq \lambda + \alpha$, no split can improve upon the leaf cost as per Corollary~\ref{cor:valid_inequality}. 
\begin{corollary}[Valid Inequality]\label{cor:valid_inequality}
If $\text{Err}(\mathcal{D}_s) \leq \lambda + \alpha$, then $(\mathcal{D}_s, d)$ can be marked terminal without expansion, without affecting optimality. Proof in Appendix~\ref{sec:admissibility}.
\end{corollary}

\subsubsection{Candidate Change Points}

Searching for an optimal decision tree is computationally demanding, and the additional expressivity of Shape Generalized Trees makes the search space prohibitively large for deep trees or large datasets. Like prior optimal tree methods, Literati therefore restricts the search to a subset of candidate change-points. More formally, we define a candidate generation function $\Phi$, which maps each data state $(\mathcal{D}_s, d)$ to a collection of per-feature candidate sets,
\begin{equation}
    \Phi(\mathcal{D}_s, d) = \bigl(U_s^{(1)}, \ldots, U_s^{(M)}\bigr), \label{eq:candidate_generator}
\end{equation}
where $U_s^{(j)}$ is a finite set of candidate change-points for feature $j$. Given $\Phi$, the actions available in our AND/OR graph become
\begin{gather}
    \mathcal{A}_{\text{split}}(\mathcal{D}_s, d) = \bigl\{ a_{j,u} : j \in \{1,\ldots,M\},\; u \in U_s^{(j)} \bigr\}, \label{eq:restricted_split} \\
    \mathcal{A}_{\text{refine}}(\mathcal{D}_s, d, j, \mathbf{k}) = \bigl\{ a_{j,\mathbf{k},u} : u \in U_s^{(j)},\; u > \max(\mathbf{k}),\; |\mathbf{k} \cup \{u\}| \leq K \bigr\}. \label{eq:restricted_refine}
\end{gather}

The generator $\Phi$ admits several instantiations. Setting $U_s^{(j)} = \mathcal{U}_s^{(j)}$ recovers the unrestricted search over all change-points. Pre-discretization, as used by most prior optimal tree methods, corresponds to a data-independent $\Phi$ that returns a fixed set of thresholds per feature regardless of the state, acting as a lookup table. These thresholds are typically feature quantiles~\citep{dl85,gosdt,van2023necessary,chaouki2025branches,demirovic2023blossom} or ensemble-derived thresholds~\citep{mctavish2022fast}.

However, such data-independent discretization can miss locally informative change-points relevant to specific decision nodes. We therefore adopt the \emph{adaptive} discretization strategy of DPDT~\citep{kohler2025breiman}. At each data state $(\mathcal{D}_s, d)$, $\Phi$ fits a CART tree on $\mathcal{D}_s$ with at most $B_d+1$ leaf nodes, yielding at most $B_d$ candidate thresholds in total, and returns its thresholds on each feature as $U_s^{(j)}$. The schedule $B_d$ is user-specified and typically increases with the remaining depth budget $d$, since states closer to the root have greater influence on tree structure; it provides a direct lever for trading solution quality against computational cost. In practice, this strategy yields substantial speedups in convergence while attaining higher training accuracy than quantile pre-discretization (Appendix~\ref{sec:prequant}) and an objective close to that over the full candidate set (Section~\ref{sec:optimization_performance}).

\paragraph{Optimality Guarantees.} For any $\Phi$ that restricts the candidate set (i.e.\ $U_s^{(j)} \subsetneq \mathcal{U}_s^{(j)}$), Literati certifies optimality only over the SGTs reachable under $\Phi$, not over the original continuous feature space, a limitation shared by all discretization-based methods~\citep{gosdt,kohler2025breiman,chaouki2025branches,mctavish2022fast}. The structure of this space, and hence the properties of the resultant optimality certificate, depends on $\Phi$. Under pre-discretization, Literati certifies optimality against all possible recompositions of the change-points, matching the guarantees provided by other methods requiring pre-discretization. Under adaptive discretization, Literati no longer provides closure under recomposition, but guarantees that its optimal regularized training objective is no greater than that of DPDT~\citep{kohler2025breiman} and CART~\citep{CART} under the matched conditions in Appendix~\ref{sec:cand_gen}.

\subsection{Anytime AO*}\label{sec:anytimeaostar}
To support anytime behaviour, we maintain a primal bound $\overline{V}(s)$ and its induced policy $\overline{\pi}$, which tracks the best solution found so far from the subgraph rooted at $s$. Data states are initialized with $\overline{V}(\mathcal{D}_s, d) = \text{Err}(\mathcal{D}_s)$ (treating the data state as terminal). For shape function states, we eagerly compute the primal bounds of the two commit successors, and set $\overline{V}$ to their sum. $\overline{V}$ is updated on expansion using a bottom-up rule analogous to Equation~\ref{eq:estimate_update}. On early termination, we follow $\overline{\pi}$ from the root until we reach either a terminal state or an unexpanded state. If the reached state is a terminal or unexpanded data state, we treat it as a leaf of the returned tree. If it is an unexpanded shape function state, we apply the commit action and treat both resulting data states as leaves. 
Since the admissible heuristic lower-bounds the value at a state, it serves as a dual bound and can be used to implement pruning strategies analogous to those used in AND/OR branch-and-bound~\cite{AObb}; If $Q(s,a) > \overline{V}(s)$, then the corresponding sub-graph is guaranteed to be sub-optimal, therefore we can prune $a$, i.e.\ $\mathcal{A}(s) := \mathcal{A}(s) \setminus \{a\}$.

While $\overline{\pi}$ provides a feasible tree at any point, we empirically observe that substantial search is required before a high-quality solution is found. This loose primal bound also limits the effectiveness of pruning, as  $Q(s,a) > \overline{V}(s)$ is rarely satisfied during the early search iterations. We trace this to two distinct bottlenecks in the selection phase, one at OR dynamics and one at AND dynamics. Each bottleneck is addressed with a targeted modification that improves anytime performance while preserving the optimality of the policy returned at the termination of AO*. We describe each of them in the following sections. Detailed pseudocode and implementation details of our approach can be found in Appendix~\ref{sec:pseudocode}.

\subsubsection{Secondary Inadmissible Heuristic for OR Selection}\label{sec:or_selection}
\begin{wrapfigure}{R}{0.56\textwidth}
\centering
\tikzset{
    datastate/.style={
        circle, draw=black, line width=2.5pt, fill=white,
        minimum size=2.5cm, align=center, font=\bfseries\Huge
    },
    shapestate/.style={
        rectangle, draw=black, line width=2pt, fill=white,
        minimum height=1.6cm, minimum width=2.2cm, align=center, font=\bfseries\Huge
    },
    refinestate/.style={
        rectangle, draw=refineColor, line width=2pt, fill=white,
        minimum height=1.6cm, minimum width=2.2cm, align=center, font=\bfseries\Huge
    },
    beststate/.style={
        rectangle, draw=commitColor, line width=4pt, fill=white,
        minimum height=1.6cm, minimum width=2.2cm, align=center, font=\bfseries\Huge
    },
    cost/.style={font=\bfseries\Huge, fill=white, inner sep=3pt},
    costbig/.style={font=\bfseries\Huge, fill=white, inner sep=3pt},
    bestedge/.style={->, -{Stealth[scale=1.3]}, line width=3pt, color=commitColor},
    refineedge/.style={->, -{Stealth[scale=1.3]}, line width=3pt, color=refineColor},
    normaledge/.style={->, -{Stealth[scale=1.3]}, line width=2pt, color=black}
}
%
%

%
%
\centering
\resizebox{\linewidth}{!}{%
\begin{tikzpicture}[scale=1, transform shape, font=\huge]

\path (9.75, 0);

\node (rootR) [datastate] at (0, 0) {$\lambda + 2\alpha$};

\node (sR1) [refinestate] at (-6, -4) {$2\alpha$};
\node (sR2) [beststate] at (0, -4) {$\alpha$};
\node (sR3) [shapestate] at (6, -4) {$\alpha$};

\draw [refineedge] (rootR) -- (sR1) node[midway, costbig, text=refineColor] {$\lambda + 3\alpha$};
\draw [bestedge] (rootR) -- (sR2) node[midway, costbig, text=commitColor] {$\lambda + 2\alpha$};
\draw [normaledge] (rootR) -- (sR3) node[midway, costbig] {$\lambda + 2\alpha$};

\def\commitOff{2.5}
\node (dL) [datastate] at (-6 - \commitOff, -9.5) {$\lambda + \alpha$};
\node (dR) [datastate] at (-6 + \commitOff, -9.5) {$\lambda + \alpha$};

\draw [normaledge] (sR1.south) -- (dL.north);
\draw [normaledge] (sR1.south) -- (dR.north);
\draw [line width=2pt, color=black] ($(sR1.south)!0.4!(dL.north)$) to[bend right=35] ($(sR1.south)!0.4!(dR.north)$);
\node [costbig, anchor=south east] at ($(sR1.south)!0.5!(dL.north) + (-0.15cm, 0)$) {$2(\lambda + \alpha)$};

\node (rR1) [shapestate] at (1, -9.5) {$\alpha$};
\node (rR2) [shapestate] at (5, -9.5) {$\alpha$};

\draw [normaledge] (sR1.south) -- (rR1.north) node[midway, costbig] {$2\alpha$};
\draw [normaledge] (sR1.south) -- (rR2.north) node[midway, costbig] {$2\alpha$};

\end{tikzpicture}
}
\caption{The blind heuristic problem. Nodes show $V$; edges show $Q$. Red: node selected this iteration; orange: node expanded in the prior iteration.}\label{fig:blind_issue}
\end{wrapfigure}
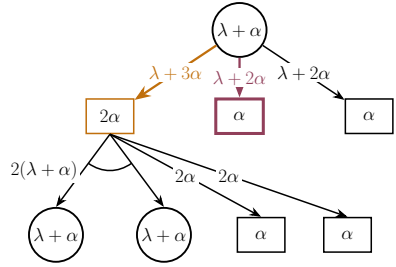

The first issue is poor selection among sibling OR nodes, which AO* cycles through in a breadth-first-like pattern due to a weak admissible heuristic. We illustrate this for a data state $s$ in Figure~\ref{fig:blind_issue}, and a similar pattern arises for shape function states. For simplicity, we assume $\text{Err}(\mathcal{D}_s) \gg \alpha$ and $\text{Err}(\mathcal{D}_s) \gg \lambda$, and that these inequalities are preserved under splits of $\mathcal{D}_s$. Such a regime is representative of the upper levels of the tree, where partitions remain coarse, and leaf errors are substantially larger than the regularization parameters, so the regularization costs dominate. When state $s$ (root node in Figure~\ref{fig:blind_issue}) is expanded, every unexpanded shape function successor $s'$ enters $G$ with $V(s') = \min(2(\lambda + \alpha),\alpha) = \alpha$, yielding identical $Q$-values of $\lambda + 2\alpha$ across all split actions. Tie-breaking among these actions is therefore arbitrary: the algorithm expands one successor, lifting its $Q$-value above the tie, so $\pi(s)$ redirects to any remaining unexplored actions, and the pattern repeats. The result is poor anytime performance as AO* touches every split candidate once before deepening into any particular subtree. 

To address this problem, we introduce a secondary, more informative but potentially inadmissible, heuristic $\widetilde{V}$ to guide selection. Each state $s$ maintains both the admissible heuristic $V(s)$ and the informative heuristic $\widetilde{V}(s)$, which is updated on expansion using a bottom-up rule analogous to Equation~\ref{eq:estimate_update}. $\widetilde{V}$ induces a selection policy $\widetilde{\pi}$ that prescribes, at each state, the action with the lowest $\widetilde{Q}$-value. During selection, AO* follows $\widetilde{\pi}$ from the root to identify the next state to expand instead of $\pi$.  Since $\widetilde{V}$ can potentially be inadmissible, when all successors in $F(s, \widetilde{\pi}(s))$ are COMPLETE, we cannot mark $s$ as COMPLETE. Instead, we revise $\widetilde{\pi}(s)$ to the action $a \in \mathcal{A}(s)$ with the lowest $\widetilde{Q}$-value among those whose successors are not all COMPLETE and continue the search to preserve optimality~\citep{hansen2007anytime}.  Pruning and termination remain governed exclusively by $\pi$: a state is marked COMPLETE only via its best partial policy under $V$, pruning conditions are unchanged, and AO* terminates when the root is COMPLETE under $\pi$. Since $\widetilde{V}$ and $\widetilde{\pi}$ only influence the order of expansion, the policy returned at termination is guaranteed to be optimal.

\paragraph{CART-lookahead.} We propose a CART-based lookahead heuristic for $\widetilde{V}$. At a data state $(\mathcal{D}_s, d)$, we fit a CART tree to $\mathcal{D}_s$ with depth budget $d$ and set $\widetilde{V}(\mathcal{D}_s, d)$ to its misclassification count plus $(\lambda + \alpha)$ per internal node (Equation~\ref{eq:objective}). For shape function states $(\mathcal{D}_s, d, j, \mathbf{k})$, we partition $\mathcal{D}_s$ via $f_j^{\mathbf{k}}$ and set $\widetilde{V}(\mathcal{D}_s, d, j, \mathbf{k}) = \widetilde{V}(\mathcal{D}_s(f_j^{\mathbf{k}})_L, d-1) + \widetilde{V}(\mathcal{D}_s(f_j^{\mathbf{k}})_R, d-1)$. Literati supports other informative heuristics, and alternatives are discussed in Appendix~\ref{sec:inadmissible_comp}.

\subsubsection{Round Robin Selection for AND Nodes}\label{sec:and_selection}

The second issue occurs at AND dynamics. When the commit action is selected at a shape function state, the algorithm must choose which of the two resulting data states $(s_L, s_R)$ to explore first. \citet{chaouki2025branches} selects the child with the worst $\overline{V}$ to prioritize improving the subtree with the inferior incumbent solution. However, when primal values are significantly imbalanced - e.g.\ $\overline{V}(s_L) \gg \overline{V}(s_R)$ or vice versa - this rule keeps returning to $s_L$ for many iterations before $\overline{V}(s_L)$ drops enough for $s_R$ to be selected. As a result, early iterations repeatedly expand one subtree while the sibling remains unexplored, even when the larger $\overline{V}$ reflects an inherently harder subproblem. Substituting $\overline{V}$ for $V$ or $\widetilde{V}$ does not resolve this, as analogous imbalances may arise. To address this, we adopt a round-robin selection policy, and when both $s_L$ and $s_R$ are unsolved, we alternate between them on successive selection calls, ensuring balanced exploration of both subtrees. If one child is already solved, selection proceeds to the remaining unsolved child. 

\section{Experimental Evaluation}\label{sec:experiments}
We evaluate Literati on 24 small-to-medium datasets ($<$500k samples) from \citet{upadhyaempowering}, which includes the QuantBnB benchmark~\cite{pmlr-v162-mazumder22a} used by \citet{kohler2025breiman} and \citet{brita2025optimal} (dataset details in Appendix~\ref{sec:add_exp_details}).  Unless otherwise specified, results are averaged over five folds with a 70/10/20 train/val/test split. We compare Literati against SGT baselines ShapeCART and ShapeTAO~\cite{upadhyaempowering}, and non-SGT baselines CART~\cite{CART}, AxTAO~\cite{TAO}, Branches~\cite{chaouki2025branches}, STreeD~\cite{van2023necessary}, LDS-DL8.5~\cite{kiossou2022time}, CADL8.5~\cite{kiossou2026generic}, ConTree~\cite{brita2025optimal} (STreeD with continuous feature support), and DPDT~\cite{kohler2025breiman}. Of these, CART, ShapeCART, AxTAO, ShapeTAO are non-optimal induction approaches. Among the optimal approaches, ConTree operates natively on continuous features, Branches, STreeD, LDS-DL8.5, and CADL8.5 require pre-discretization, while DPDT and Literati use adaptive discretization. Furthermore, LDS-DL8.5~\cite{kiossou2022time} and CADL8.5~\cite{kiossou2026generic} both prioritize anytime performance. For runs that fail to produce any solution within the time limit (only DPDT), we impute the accuracy of a single-leaf tree. For runs that exhaust the 64~GB memory limit (only observed for Branches), we report the accuracy and runtime achieved at the time of the crash.

\subsection{Optimization Performance}\label{sec:optimization_performance}

\begin{table}[!t]
\centering
\caption{Training accuracy, runtime, and optimality proof rate across depth budgets averaged across all datasets evaluated, organized into Non-Optimal and Optimal approaches, with the latter grouped by discretization strategy. The dashed lines separate Literati from baselines. Bold indicates the best value per depth for training accuracy and runtime.}
\label{tab:OptimalityTable}
\resizebox{\linewidth}{!}{%
\begin{tabular}{@{}rcccccrrrrcccc@{}}
\toprule
 &  & \multicolumn{4}{c}{$\uparrow$ Train Acc. (\%)} & \multicolumn{4}{c}{$\downarrow$ Runtime (s)} & \multicolumn{4}{c}{$\uparrow$ Proof Obtained (\%)} \\ \midrule
Model & $K$ & $D$=3 & $D$=4 & $D$=5 & $D$=6 & \multicolumn{1}{c}{$D$=3} & \multicolumn{1}{c}{$D$=4} & \multicolumn{1}{c}{$D$=5} & \multicolumn{1}{c}{$D$=6} & $D$=3 & $D$=4 & $D$=5 & $D$=6 \\\midrule
\multicolumn{14}{c}{\textbf{Non-Optimal}} \\
CART & 1 & 79.60 & 82.70 & 85.65 & 87.64 & \textbf{0.08} & \textbf{0.11} & \textbf{0.13} & \textbf{0.16} & --- & --- & --- & --- \\
AxTAO & 1 & 80.25 & 84.04 & 86.55 & 88.46 & 5.41 & 8.05 & 10.98 & 13.87 & --- & --- & --- & --- \\
ShapeCART & 3 & 80.43 & 83.64 & 86.34 & 88.52 & 1.82 & 1.96 & 3.03 & 3.84 & --- & --- & --- & --- \\
ShapeTAO & 3 & 81.28 & 85.07 & 87.33 & 89.49 & 3.85 & 8.07 & 17.26 & 29.84 & --- & --- & --- & --- \\
\multicolumn{14}{c}{\textbf{Optimal (No Discretization)} } \\
ConTree & 1 & 84.12 & 86.68 & 86.13 & 84.35 & 388.18 & 1869.40 & 2542.97 & 2700.91 & 95.8 & 54.2 & 33.3 & 25.0 \\
\multicolumn{14}{c}{\textbf{Optimal (Quantile Pre-Discretization)}} \\
Branches & 1 & 81.17 & 79.53 & 79.11 & 79.07 & 95.30 & 155.23 & 158.90 & 152.59 & 62.5 & 6.7 & 5.8 & 8.3 \\
STreeD & 1 & 83.50 & 86.57 & 88.44 & 88.26 & 169.90 & 413.52 & 1693.70 & 2445.12 & 99.2 & 91.7 & 65.0 & 37.5 \\
LDS-DL8.5 & 1 & 83.50 & 86.95 & 89.09 & 90.50 & 620.79 & 2459.50 & 3300.01 & 3212.81 & 87.5 & 37.5 & 8.3 & 12.5 \\
CADL8.5 & 1 & 83.50 & 86.14 & 87.29 & 84.12 & 315.62 & 1179.99 & 2481.14 & 3154.14 & 91.7 & 79.2 & 38.3 & 12.5 \\
\multicolumn{14}{c}{\textbf{Optimal (Adaptive Discretization)}} \\
DPDT & 1 & 83.56 & 87.03 & 77.16 & 71.52 & 46.57 & 488.23 & 2387.41 & 3217.40 & 100.0 & 95.8 & 45.8 & 12.5 \\\hdashline \noalign{\smallskip}
\multirow{3}{*}{Literati} & 1 & 83.56 & 87.28 & 89.91 & 92.28 & 18.22 & 259.14 & 1156.29 & 2587.19 & 100.0 & 97.5 & 87.5 & 36.7 \\
 & 2 & 84.03 & 87.75 & 90.36 & 92.58 & 60.82 & 851.76 & 2672.71 & 2817.87 & 100.0 & 87.5 & 33.3 & 22.5 \\
 & 3 & \textbf{84.13} & \textbf{87.90} & \textbf{90.46} & \textbf{92.69} & 162.93 & 1393.71 & 2832.73 & 2782.09 & 100.0 & 82.5 & 25.0 & 24.2 \\ \bottomrule
\end{tabular}%
}
\end{table}

In our first experiment, we assess Literati's training performance across a range of depth and change-point budgets. We run Literati and our baselines on all datasets with a one-hour limit over depth budgets $D \in \{3, \ldots, 6\}$. For Literati, we sweep shape complexity $K \in \{1, 2, 3\}$ (where $K = 1$ recovers an axis-aligned threshold tree). ShapeTAO and ShapeCART are evaluated at $K = 3$, the most generous setting. For methods requiring pre-discretization, we preprocess features to 16 quantiles. For DPDT and Literati, we set the number of candidate thresholds to 16 for all depths ($B_d = 16\ \forall d = 1,\ldots, D$). Appendix~\ref{sec:add_exp_details} provides further details on the experimental setup. Table~\ref{tab:OptimalityTable} reports average training accuracy and runtime. While proof rates are not directly comparable across different discretization methods, we report them where applicable for completeness. Per-dataset results and per-depth critical difference diagrams appear in Appendix~\ref{sec:opt_full_results}.

\paragraph{Optimization Capabilities.} To isolate Literati's optimization capabilities from the expressiveness of shape functions, we first consider the $K=1$ setting, where Literati produces axis-aligned threshold trees. Literati achieves the lowest runtime among optimal methods for $D \in \{3,4,5\}$, except for Branches, whose lower average runtimes stem from frequent early termination due to out-of-memory errors. Literati also achieves the highest training accuracy for $D \in \{4,5,6\}$ across all evaluated methods. While ConTree attains the highest training accuracy at the shallowest depth ($D=3$), Literati significantly outperforms it for $D \in \{4,5,6\}$. Literati also substantially outperforms all prediscretized approaches (Branches, STreeD, LDS-DL8.5, and CADL8.5) in both training accuracy and runtime across all depths and $K$ values. In particular, Literati achieves significantly higher training accuracy than Branches, the other AO*-based approach, which frequently exhausts the 64 GB memory limit, highlighting the effectiveness of Literati's pruning strategy. Against DPDT, Literati matches its training accuracy exactly at $D=3$ while achieving a lower runtime. Although both operate over the same search space when $K=1$, Literati achieves a higher optimality proof rate, which we attribute to its AO*-based search reducing the number of subproblems that must be solved. Both methods also support depth-dependent $B_d$ schedules that trade solution quality for runtime, and Literati's advantage over DPDT persists under a lighter adaptive discretization regime (Appendix~\ref{sec:opt_full_results}).

\paragraph{Effect of Shape Complexity.} Increasing $K$ substantially improves training accuracy: $K=1$ to $K=2$ yields $\approx 0.4\%$ across all depths, with $K=2$ to $K=3$ contributing a smaller but consistent gain. Consequently, $K=3$ achieves the highest accuracy at every depth, exceeding the SGT baselines with similar $K$ (ShapeCART, ShapeTAO) and all optimal tree methods. 

\paragraph{Anytime Performance.} We observe that Literati's training accuracy continues to improve with depth at every $K$, even when proof rates fall, reflecting strong anytime behaviour. STreeD and LDS-DL8.5 exhibit a similar trend, but the gap between Literati and these two models widens with increasing depth. ConTree, CADL8.5, and Branches exhibit performance degradation as depth increases. DPDT does not return intermediate solutions upon timeout, resulting in poor performance at higher depths.

\subsubsection{Selection Strategy Ablation}\label{sec:anytime_ao_ablation}

\begin{wrapfigure}[16]{r}{0.6\textwidth}
    \centering
    \includegraphics[width=\linewidth]{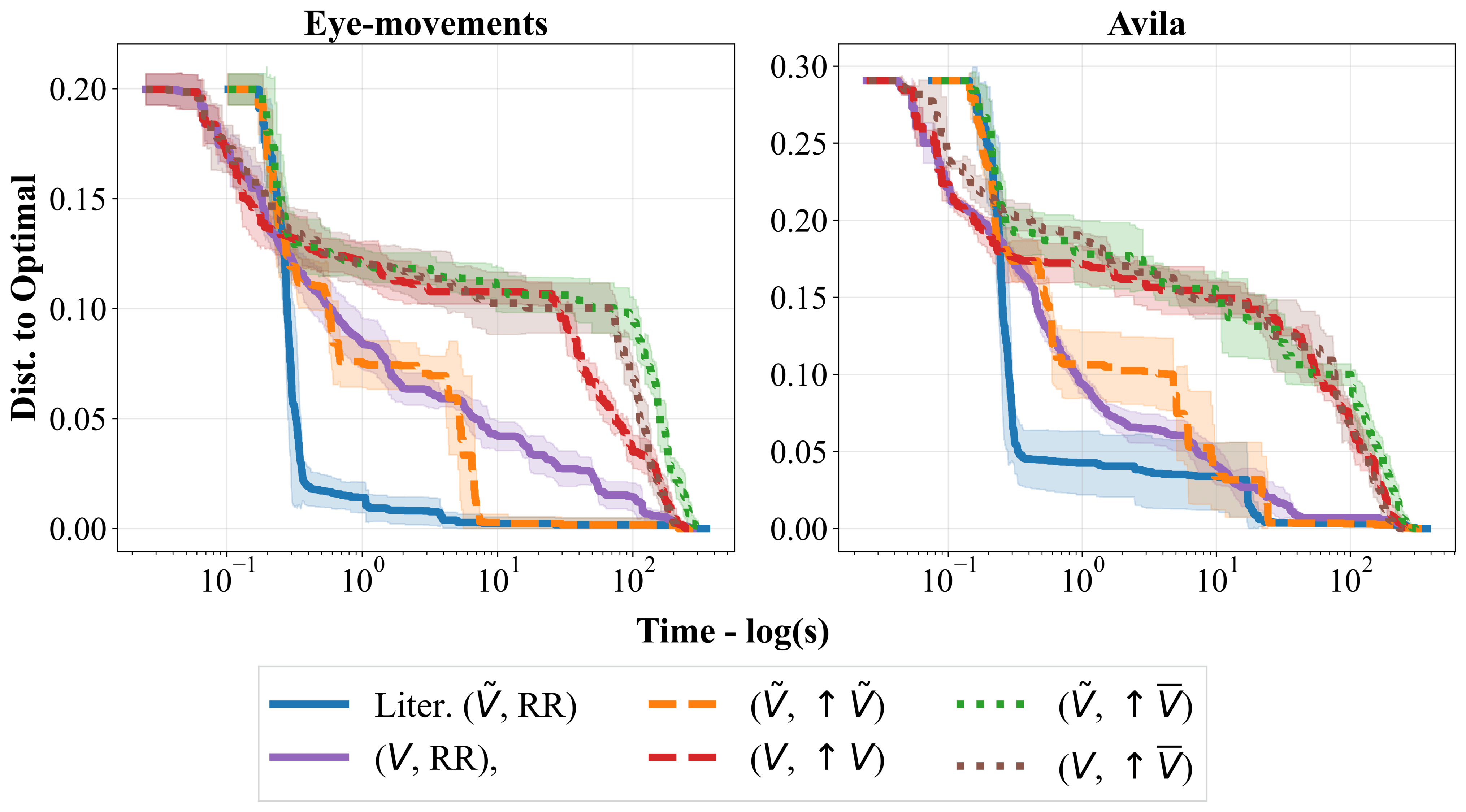}
    \caption{Dist.\ to Opt.\ Solution. Legend key: (OR Selection, AND Selection).}
    \label{fig:opt_distance}
\end{wrapfigure}


In this section, we provide an ablation of our selection strategies to isolate their contributions to anytime performance. Specifically, we ablate our OR selection method (best informative heuristic $\widetilde{V}$) by selecting using the admissible heuristic $V$. Additionally, we ablate our AND selection strategy (round-robin (RR)) and consider two alternatives: (1) selecting the child with the worst primal value~\citep{chaouki2025branches} ($\uparrow \overline{V}$) and (2) selecting the child with the worst value under the OR heuristic in use ($\uparrow \widetilde{V}$ or $\uparrow V$). All configurations are evaluated on five representative datasets (Eye-Movements, Avila, Page, Fault, and Rice) across five folds with $D=6$, $K = 3$ under a one-hour time limit. Figure~\ref{fig:opt_distance} plots the distance to the optimal solution over time for Eye-Movements and Avila, and results for the other three datasets, plus distances across iterations, can be found in Appendix \ref{sec:add_inad_results}.

We observe that both the informative inadmissible heuristic and the round-robin AND selection meaningfully contribute to anytime performance, and that Literati achieves the best overall anytime performance when combining them.

\subsection{Generalization}\label{sec:generalization}
In addition to the optimization capabilities of Literati, we are interested in how well our approach performs on out-of-sample data points. For all methods, we tune hyperparameters using a 200-trial random search and select the trial with the best average validation accuracy across folds. To reflect practical training constraints, each model is given a 15-minute training budget per fold. Full details of each method's hyperparameter search space are provided in Appendix \ref{sec:hp}. Across datasets and folds, Table~\ref{tab:gen_table} reports average test accuracy, and Figure~\ref{fig:cdd} shows the corresponding critical difference diagram. Per dataset results are provided in Appendix~\ref{sec:per_dataset_results}.

\begin{wrapfigure}[23]{r}{0.5\textwidth}
    \centering
    \includegraphics[width=\linewidth]{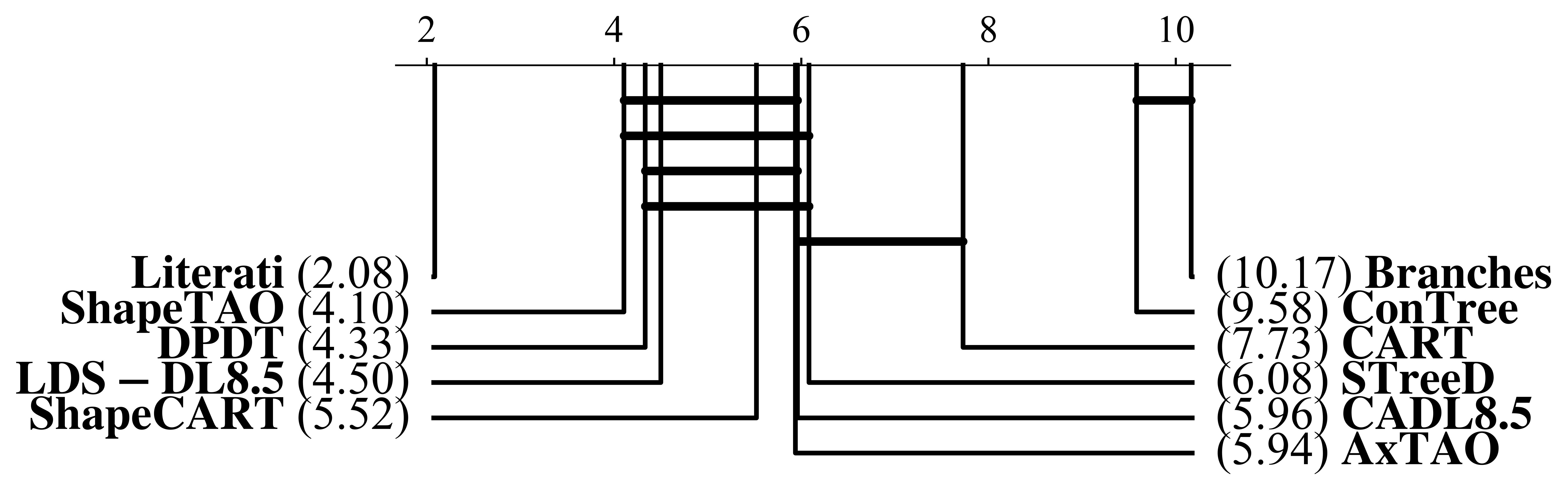}
    \captionof{figure}{CDD of test accuracy. Axis shows each method's mean rank across datasets and folds; bars indicate statistical equivalence ($p \geq 0.05$, Wilcoxon).}
    \label{fig:cdd}
    \captionof{table}{Average Test Accuracy Results.}
    \label{tab:gen_table}


\begin{tabular}{@{}cc@{}}
\toprule
Model     & Test Acc.\ (\%) \\ \midrule
CART      & 87.6          \\
AxTAO     & 88.5          \\
ShapeCART & 89.1          \\
ShapeTAO  & \underline{89.6} \\
Branches  & 83.7          \\
STreeD    & 88.5          \\
LDS-DL8.5 & \underline{89.6} \\
CADL8.5   & 88.6          \\
ConTree   & 85.5          \\
DPDT      & 89.3          \\
Literati  & \textbf{90.6} \\ \bottomrule
\end{tabular}

\end{wrapfigure}
We observe that Literati achieves an average test accuracy 1.0\% higher than the second-place approach, ShapeTAO, and a statistically significantly higher average rank than all evaluated baseline approaches.

\section{Conclusion}
We introduced Literati, the first algorithm for optimal Shape Generalized Tree induction. By formulating SGT learning as an AND/OR graph search that jointly optimizes tree structure and shape function complexity, Literati provides optimality guarantees while retaining anytime performance through two modifications to AO* selection: an inadmissible secondary heuristic at OR nodes and round-robin exploration at AND nodes. Across 24 benchmark datasets, Literati achieves superior performance over all evaluated optimal-tree and greedy SGT baselines, establishing heuristic search as a strong foundation for SGT induction. Future work includes exploring additional admissible and informative heuristics within our AO* framework. 
\paragraph{Limitations.} As a tree-based model, Literati, like other SGTs, inherits intrinsic constraints of traditional decision trees, most notably that they are designed primarily for tabular data. Additionally, like other optimal tree induction approaches, Literati will struggle to prove optimality on datasets with a large number of samples within short time budgets. Nevertheless, the strong anytime performance of Literati makes it a valuable and practical tool for data science workflows. 
\paragraph{Broader Impact.} As ML increasingly governs high-stakes decisions in healthcare~\cite{stiglic2020interpretability,5593711}, finance~\cite{zurada2010could,TANG2024102088}, and manufacturing~\cite{sungsu2017decision,HEYDARBAKIAN20222834}, the need for models that balance accuracy and interpretability is pressing. Literati contributes to this need by producing inherently interpretable trees that outperform standard tree-based baselines in our experiments.

\section{Acknowledgments and Disclosure of Funding}
This research was supported by grant number DSI-CGY3R1P18 from the Data Sciences Institute at the University of Toronto. The authors also gratefully acknowledge funding from the Natural Sciences and Engineering Research Council of Canada (NSERC).

\bibliographystyle{unsrtnat}
\bibliography{bibfile}
\newpage
\appendix
\raggedbottom
\renewcommand\thefigure{\thesection.\arabic{figure}}

\section{Theoretical Results}
\subsection{Validity of AND/OR formulation}\label{sec:ao_validity_proof}
Let $\mathcal{L}_D^*(\mathcal{D})$ denote the optimum of Equation~\ref{eq:objective}. We first establish that the recurrences in Equations~\ref{eq:node_full}--\ref{eqn:shaperecurrence} compute $\mathcal{L}_D^*(\mathcal{D})$ (Lemma~\ref{lem:recurrence_validity}), then show that the Bellman equation on $(\mathcal{S}, \ell, \mathcal{A}, F, c)$ is the exact instantiation of these recurrences (Lemma~\ref{lem:bellman_instantiation}). Composing the two yields $V^*(s_0) = \mathcal{L}_D^*(\mathcal{D})$, and the greedy policy $\pi^*$ traces out a corresponding optimal tree.

\begin{lemma}[Validity of the Recurrence]\label{lem:recurrence_validity}
The function $\mathcal{L}^*$ defined by Equations~\ref{eq:node_full}--\ref{eqn:shaperecurrence} satisfies $\mathcal{L}_D^*(\mathcal{D}) = \min_{T \in \mathcal{T}} \{\emph{\text{Err}}(T, \mathcal{D}) + \sum_{v \in \mathcal{I}(T)}[\lambda + \alpha \mathbf{K}(f_v)] : \mathbf{K}(f_v) \leq K\ \forall v,\ \mathbf{D}(T) \leq D \}$.
\end{lemma}

\begin{proof}
The regularized loss is additive across the subtrees of an internal node, and both feasibility constraints are hereditary: every subtree of a feasible tree is itself feasible. Specifically, $\mathbf{K}(f_v) \leq K$ is local to node $v$, and the depth budget at depth $d$ is at most $D - d$, which is exactly what the recurrence's depth argument tracks. Since an optimal tree must be composed of optimal subtrees, the minimization over trees decomposes into nested minimizations over subtrees, reducing Equation~\ref{eq:objective} to Equation~\ref{eq:node_problem}.

We now show that the incremental recurrence in Equations~\ref{eq:node_full}--\ref{eqn:shaperecurrence} enumerates the same feasible shape functions as the inner minimization of Equation~\ref{eq:node_problem} at identical cost. 
Every feasible shape function for node $v$ on feature $j$ is determined by its sorted change-point set $\mathbf{k} \subseteq \mathcal{U}_v^{(j)}, |\mathbf{k}| \leq K$ along with a choice of which child each interval routes to. Given $j$ and $\mathbf{k}$, there are exactly two such functions, $f_j^{\mathbf{k}}$ and $\bar f_j^{\mathbf{k}}$, which partition $\mathcal{D}_v$ into the same two groups but route them to opposite children. We adopt the convention $f_j^{\mathbf{k}}(x) = 0$ for $x < \min(\mathbf{k})$, and enumerate only $f_j^{\mathbf{k}}$. This loses no optimal trees: $f_j^{\mathbf{k}}$ and $\bar f_j^{\mathbf{k}}$ produce the same two data subsets up to a swap of $\mathcal{D}_v(f_j^{\mathbf{k}})_L$ and $\mathcal{D}_v(f_j^{\mathbf{k}})_R$, and the recurrence $\mathcal{L}_{d-1}^*(\mathcal{D}_v(f_j^{\mathbf{k}})_L) + \mathcal{L}_{d-1}^*(\mathcal{D}_v(f_j^{\mathbf{k}})_R)$ is symmetric in its two arguments, so both achieve identical cost. A set of change-points $\mathbf{k} = \{u_1,u_2,\ldots,u_m\}$ is constructed via the sequence $\text{split}(j, u_1) \to \text{refine}(u_2) \to \cdots \to \text{refine}(u_m) \to \text{commit}$, accumulating cost $(\lambda + \alpha) + (m-1)\alpha + 0 = \lambda + \alpha m$, which matches $\lambda + \alpha \mathbf{K}(f_v)$. Given $K\geq 1$, constant shape functions ($\mathbf{K}(f_j^\mathbf{k}) = 0$) are not enumerated. This is without loss of optimality, since a constant $f_v$ sends all of $\mathcal{D}_v$ to one child and yields a degenerate split with cost $\lambda + \mathcal{L}_{d-1}^*(\mathcal{D}_v) \geq \mathcal{L}_d^*(\mathcal{D}_v)$, which is dominated by either the leaf action or any nontrivial split. The ordering constraint $u > \max(\mathbf{k})$ on refine does not exclude any shape function; it only prevents symmetric shape functions. Hence the minimization in Equations~\ref{eq:node_full}--\ref{eqn:shaperecurrence} corresponds to the minimization over feasible $f_v$ in Equation~\ref{eq:node_problem}, with identical costs. \qedhere
\end{proof}

\begin{lemma}[AND/OR Graph Translation]\label{lem:bellman_instantiation}
For every data state $(\mathcal{D}_s, d) \in \mathcal{S}$ and shape function state $(\mathcal{D}_s, d, j, \mathbf{k}) \in \mathcal{S}$,
\begin{gather}
    V^*(\mathcal{D}_s, d) = \mathcal{L}_d^*(\mathcal{D}_s), \qquad V^*(\mathcal{D}_s, d, j, \mathbf{k}) = \mathcal{L}_d^*(\mathcal{D}_s, j, \mathbf{k}).
\end{gather}
\end{lemma}

\begin{proof}
We show that $V^*$ satisfies the same recurrence as $\mathcal{L}^*$ over the same domain with the same base cases. 

\emph{Terminal states.} For $\overline{(\mathcal{D}_s, d)} \in \ell$, Equation~\ref{eq:bellman} gives $V^*(s) = \text{Err}(\mathcal{D}_s)$, matching the base case of Equation~\ref{eq:node_full}. This covers both terminal conditions: depth exhaustion ($d = 0$) and label purity ($\text{Err}(\mathcal{D}_s) = 0$). For pure nodes the equality reads $V^*(s) = 0 = \mathcal{L}_d^*(\mathcal{D}_s)$, which agrees with the base case of Equation~\ref{eq:node_full} since $\text{Err}(\mathcal{D}_s) = 0$ dominates any split cost $\lambda + \alpha \geq 0$.

\emph{Data states.} For a non-terminal data state $s = (\mathcal{D}_s, d)$, the leaf action gives $Q^*(s, \bar a) = c(s, \bar a) + V^*(\overline{(\mathcal{D}_s, d)}) = 0 + \text{Err}(\mathcal{D}_s)$. Each split action gives $Q^*(s, a_{j,u}) = c(s, a_{j,u}) + V^*(\mathcal{D}_s, d, j, \{u\}) = \lambda + \alpha + V^*(\mathcal{D}_s, d, j, \{u\})$. Equation~\ref{eq:bellman} then yields
\begin{gather}
    V^*(\mathcal{D}_s, d) = \min\!\left(\text{Err}(\mathcal{D}_s),\; \min_{j \in \{1,\ldots,M\}}\min_{u \in \mathcal{U}_s^{(j)}} \lambda + \alpha + V^*(\mathcal{D}_s, d, j, \{u\})\right),
\end{gather}
which is Equation~\ref{eq:node_full}.

\emph{Shape function states.} For $s = (\mathcal{D}_s, d, j, \mathbf{k})$, the commit action transitions to the AND pair $((\mathcal{D}_s(f_j^{\mathbf{k}})_L, d-1), (\mathcal{D}_s(f_j^{\mathbf{k}})_R, d-1))$, and AND dynamics sum the children, giving $Q^*(s, a_{\text{commit}}) = 0 + V^*(\mathcal{D}_s(f_j^{\mathbf{k}})_L, d-1) + V^*(\mathcal{D}_s(f_j^{\mathbf{k}})_R, d-1)$. Each refine action gives $Q^*(s, a_{j,\mathbf{k},u}) = \alpha + V^*(\mathcal{D}_s, d, j, \mathbf{k} \cup \{u\})$. Equation~\ref{eq:bellman} yields
\begin{gather}
\resizebox{0.9\columnwidth}{!}{$\displaystyle
    V^*(\mathcal{D}_s, d, j, \mathbf{k}) = \min\!\Bigg( V^*(\mathcal{D}_s(f_j^{\mathbf{k}})_L, d-1) + V^*(\mathcal{D}_s(f_j^{\mathbf{k}})_R, d-1),\; \alpha + \min_{\substack{u \in \mathcal{U}_s^{(j)},\, u > \max(\mathbf{k}),\\ |\mathbf{k} \cup \{u\}| \leq K}} V^*(\mathcal{D}_s, d, j, \mathbf{k} \cup \{u\})\Bigg),
$}
\end{gather}
which is Equation~\ref{eqn:shaperecurrence}.

Hence $V^*$ and $\mathcal{L}^*$ satisfy identical recurrences with identical base cases, and they are well defined on the same finite domain (terminated by the depth budget and purity conditions). It follows that $V^* = \mathcal{L}^*$ pointwise on both state types. \qedhere
\end{proof}

\paragraph{Proof of Proposition~\ref{prop:validity}.}
At the root $s_0 = (\mathcal{D}, D)$, Lemma~\ref{lem:bellman_instantiation} gives $V^*(s_0) = \mathcal{L}_D^*(\mathcal{D})$, and Lemma~\ref{lem:recurrence_validity} gives that $\mathcal{L}_D^*(\mathcal{D})$ equals the minimum of Equation~\ref{eq:objective}. Hence $V^*(s_0)$ equals the optimum of the sparse SGT induction problem. As a result, the optimal tree $T^*$ can be obtained by following $\pi^*$ from the root node $s_0$ until terminal states are reached. More specifically, at each reachable data state $s$, if $\pi^*(s) = \bar a$, the corresponding node in $T^*$ is a leaf labeled by the majority class of $\mathcal{D}_s$; if $\pi^*(s) = a_{j,u}$, the node is internal with feature $j$, and $\pi^*$ is followed through the chain of shape function states until $\pi^*$ selects commit at some change-point set $\mathbf{k}^*(s)$; the branching rule at the node is then $f_j^{\mathbf{k}^*(s)}$, and the construction recurses on the two child data states. The extracted $T^*$ is feasible for Equation~\ref{eq:objective} by construction: the depth constraint $\mathbf{D}(T^*) \leq D$ holds because $d$ strictly decreases by one on every commit transition and $d = D$ at $s_0$, so no path from the root to a leaf in $T^*$ exceeds depth $D$; the complexity constraint $\mathbf{K}(f_j^{\mathbf{k}^*(s)}) \leq K$ holds because the refine action is guarded by $|\mathbf{k}| \leq K$, so $|\mathbf{k}^*(s)| \leq K$ at every internal node. Hence $T^*$ is optimal. \qed

\subsection{Admissibility of Heuristic}\label{sec:admissibility}

We prove by induction on the iterations of AO* that $V(s) \leq V^*(s)$ for every state $s$ in the explicit search graph. We first establish that the initial heuristic values satisfy admissibility (Lemma~\ref{lem:init_admissible}), then show that the bottom-up updates in Equation~\ref{eq:estimate_update} preserve admissibility (Lemma~\ref{lem:update_admissible}). The proposition then follows by induction, and the corollary on the valid inequality is immediate.

\begin{lemma}[Initial Admissibility]\label{lem:init_admissible}
The initialization $V(s) = V(\mathcal{D}_s, d) = \min(\text{Err}(\mathcal{D}_s),\, \lambda + \alpha)$ at non-terminal data states and $V(s) = V(\mathcal{D}_s, d, j, \mathbf{k}) = \min\bigl(\alpha,\, \min(\text{Err}(\mathcal{D}_s(f_j^{\mathbf{k}})_L),\, \lambda + \alpha) + \min(\text{Err}(\mathcal{D}_s(f_j^{\mathbf{k}})_R),\, \lambda + \alpha)\bigr)$ at shape function states satisfies $V(s) \leq V^*(s)$.
\end{lemma}

\begin{proof}
All transition costs are non-negative ($\lambda, \alpha \geq 0$ and commit and leaf costs are $0$), and the terminal values $\text{Err}(\mathcal{D}_s) \geq 0$. Unrolling the Bellman recursion in Equation~\ref{eq:bellman}, $V^*(s)$ is a sum of non-negative quantities, so $V^*(s) \geq 0$ for every state $s$. For a terminal state, by definition $V^*(s) = V(s)$. 

For a non-terminal data state $s = (\mathcal{D}_s, d)$, Equation~\ref{eq:bellman} gives
\begin{gather}
    V^*(s) \;=\; \min\!\left(Q^*(s, \bar a),\; \min_{j \in \{1,\dots,M\},\, u \in \mathcal{U}^{(j)}_s} Q^*(s, a_{j,u})\right).
\end{gather}
The leaf action transitions to the terminal state $\overline{(\mathcal{D}_s, d)}$ at zero cost, so $Q^*(s, \bar a) = \text{Err}(\mathcal{D}_s)$. Each split action incurs cost $\lambda + \alpha$ and transitions to a shape function state, so $Q^*(s, a_{j,u}) = \lambda + \alpha + V^*(\mathcal{D}_s, d, j, \{u\}) \geq \lambda + \alpha$. Therefore
\begin{gather}
    V^*(s) \;\geq\; \min\!\left(\text{Err}(\mathcal{D}_s),\, \lambda + \alpha\right) \;=\; V(s).
\end{gather}

For a shape function state $s=(\mathcal{D}_s,d,j,\mathbf{k})$, we can either commit $f_j^{\mathbf{k}}$ to obtain the two data states $s_L=(\mathcal{D}_s(f_j^{\mathbf{k}})_L,d-1)$ and $s_R=(\mathcal{D}_s(f_j^{\mathbf{k}})_R,d-1)$, or refine and pay a cost of $\alpha$. By the data-state bound established above, the cost of committing satisfies
\begin{equation*}
\begin{aligned}
    Q^*(s,a_{\mathrm{commit}})&=V^*(s_L)+V^*(s_R)\\
    &\geq\sum_{b\in\{L,R\}}\min\bigl(\mathrm{Err}(\mathcal{D}_s(f_j^{\mathbf{k}})_b),\lambda+\alpha\bigr).
\end{aligned}
\end{equation*}
This bound also holds for terminal children, whose values equal their leaf errors. Each refine action incurs a cost of $\alpha$ and transitions to a state with non-negative optimal cost, so $Q^*(s,a_{j,\mathbf{k},u})\geq\alpha$. Therefore,
\begin{equation*}
    V^*(s)\geq\min\Bigl(\alpha,\sum_{b\in\{L,R\}}\min\bigl(\mathrm{Err}(\mathcal{D}_s(f_j^{\mathbf{k}})_b),\lambda+\alpha\bigr)\Bigr)=V(s).
\end{equation*}
If no refine action is available, commit is the only action and the same bound holds.
\end{proof}
\begin{lemma}[Update Admissibility]\label{lem:update_admissible}
Fix a non-terminal state $\widetilde{s}$ and suppose $V(s') \leq V^*(s')$ for every $s' \in F(\widetilde{s}, a)$ and every $a \in \mathcal{A}(\widetilde{s})$. Then the update in Equation~\ref{eq:estimate_update} yields $V(\widetilde{s}) \leq V^*(\widetilde{s})$.
\end{lemma}

\begin{proof}
By the update and the hypothesis,
\begin{gather}
\begin{aligned}
    V(\widetilde{s}) \;&=\; \min_{a \in \mathcal{A}(\widetilde{s})} \left[\, c(\widetilde{s}, a) + \sum_{s' \in F(\widetilde{s}, a)} V(s') \,\right] \\
    &\leq\; \min_{a \in \mathcal{A}(\widetilde{s})} \left[\, c(\widetilde{s}, a) + \sum_{s' \in F(\widetilde{s}, a)} V^*(s') \,\right] \;=\; V^*(\widetilde{s}). \qedhere
\end{aligned}
\end{gather}
\end{proof}

\paragraph{Proof of Proposition~\ref{prop:admissibility}.}
We proceed by induction on the iteration count of AO*. At initialization, the search graph contains only the root $s_0 = (\mathcal{D}, D)$, which satisfies $V(s_0) \leq V^*(s_0)$ by Lemma~\ref{lem:init_admissible}.

Assume $V(s) \leq V^*(s)$ for every state $s$ in the partial search graph $G$ at the start of an iteration. During an iteration, AO* expands some non-terminal state $s$ and adds the successors of every action $a \in \mathcal{A}(s)$ to $G$ with their initial heuristic values; these satisfy admissibility by Lemma~\ref{lem:init_admissible}. AO* then updates $V$ for $s$ and its ancestors in bottom-up order via Equation~\ref{eq:estimate_update}. At the time each update is performed, every successor referenced in the update is admissible: either it was admissible before iteration $t$ (inductive hypothesis), it was just added with an admissible initial value (Lemma~\ref{lem:init_admissible}), or it was itself updated earlier in the bottom-up sweep and is admissible by Lemma~\ref{lem:update_admissible}. A further application of Lemma~\ref{lem:update_admissible} then gives admissibility at $\widetilde{s}$. Iterating over all ancestors preserves admissibility throughout $G$. \qed

\paragraph{Proof of Corollary~\ref{cor:valid_inequality}.}
Suppose $\text{Err}(\mathcal{D}_s) \leq \lambda + \alpha$ at a data state $s = (\mathcal{D}_s, d)$. For every split action,
\begin{gather}
    Q^*(s, a_{j,u}) \;=\; \lambda + \alpha + V^*(\mathcal{D}_s, d, j, \{u\}) \;\geq\; \lambda + \alpha \;\geq\; \text{Err}(\mathcal{D}_s) \;=\; Q^*(s, \bar a),
\end{gather}
so the leaf action is optimal and $V^*(s) = \text{Err}(\mathcal{D}_s)$.  \qed


\subsection{Upper Bound on Optimal Regularized Training Loss}\label{sec:cand_gen}

\begin{proposition}[Dominance over DPDT]\label{prop:dominance}
Suppose Literati and DPDT~\citep{kohler2025breiman} are applied to the same dataset $\mathcal{D}$ under the following matched conditions: (i) identical depth budget $D$; (ii) DPDT uses the per-internal-node complexity penalty $C(T) = |\mathcal{I}(T)|$ with regularization weight $\alpha_{\mathrm{DPDT}} = \lambda + \alpha$; and (iii) both methods use identical CART-based adaptive change-point candidate sets at every data state~\cite{kohler2025breiman} with identical $B_d$ schedules. Let $\mathcal{L}^*_{\mathrm{Liter.}}$ and $\mathcal{L}^*_{\mathrm{DPDT}}$ denote the optimal objective values attained by Literati and DPDT, respectively. Then $\mathcal{L}^*_{\mathrm{Liter.}} \leq \mathcal{L}^*_{\mathrm{DPDT}},$ with equality whenever $K = 1$. 
\end{proposition}
Throughout this proof, we denote $\mathcal{L}_{\mathrm{Liter}}(T)$ for the Literati objective in Equation~\ref{eq:objective} and $\mathcal{L}_{\mathrm{DPDT}}(T)$ for the DPDT objective under the node-count complexity penalty variant of \citet[Eq.~1]{kohler2025breiman}. We adopt two bookkeeping conventions to align the two formulations. First, DPDT is formulated as a maximization of expected return $J_\alpha(\pi)$, whereas Literati is formulated as a minimization of regularized loss; since $\mathcal{L}_{\mathrm{DPDT}}(T) = -J_\alpha(\pi)$ for the policy $\pi$ inducing $T$ \citep[Proposition 4.1]{kohler2025breiman}, we negate the DPDT return and work with the resulting loss throughout. Second, DPDT uses a normalized misclassification rate $\frac{1}{N}\sum_{i=1}^{N} \ell(y_i, T(x_i))$, whereas Literati uses the unnormalized count $\mathrm{Err}(T, \mathcal{D})$; since $1/N$ is a positive constant scaling factor that does not affect the argmin, we work with the unnormalized convention for both objectives, with the understanding that the DPDT regularization weight is rescaled accordingly. These conventions place the two objectives in a common form; the precise correspondence between their regularization terms is established below.

We first show that for any axis-aligned threshold tree $T$, the Literati objective with $K=1$ and the DPDT objective with matched regularization weight $\alpha_{\mathrm{DPDT}} = \lambda + \alpha$ coincide. Writing both objectives as a sum of an error term and a per-internal-node complexity, we obtain the following:
\begin{gather}
    \mathcal{L}_{\mathrm{Liter}}(T) = \mathrm{Err}(T, \mathcal{D}) + \sum_{v \in \mathcal{I}(T)} \left[\lambda + \alpha\, \mathbf{K}(f_v)\right], \\
    \mathcal{L}_{\mathrm{DPDT}}(T) = \mathrm{Err}(T, \mathcal{D}) + \sum_{v \in \mathcal{I}(T)} \alpha_{\mathrm{DPDT}}.
\end{gather}
Under $K = 1$, the constraint $\mathbf{K}(f_v) \leq K$ combined with the requirement that $f_v$ be non-constant at any internal node (else $v$ would not induce a split) pins $\mathbf{K}(f_v) = 1$ at every $v \in \mathcal{I}(T)$. The Literati per-node penalty thus reduces to $\lambda + \alpha \cdot 1 = \lambda + \alpha$, and applying the matching condition $\alpha_{\mathrm{DPDT}} = \lambda + \alpha$ yields
\begin{gather}
    \sum_{v \in \mathcal{I}(T)} \left[\lambda + \alpha\, \mathbf{K}(f_v)\right] = \sum_{v \in \mathcal{I}(T)} (\lambda + \alpha) = \sum_{v \in \mathcal{I}(T)} \alpha_{\mathrm{DPDT}}.
\end{gather}
The error terms agree by construction: a Literati node with change-point set $\mathbf{k} = \{u\}$ on feature $j$ routes samples by $\mathbf{1}[x_j > u]$, which is the DPDT threshold split $x_j \leq u$ up to a left-right relabeling. Routing-equivalent trees have identical $\mathrm{Err}(\cdot, \mathcal{D})$. Combining the matched error and complexity terms gives
\begin{gather}
    \mathcal{L}_{\mathrm{Liter}}^{K=1}(T) = \mathcal{L}_{\mathrm{DPDT}}(T) \qquad \forall\, T \in \mathcal{T}_{\mathrm{DPDT}}. \label{eq:stage1}
\end{gather}
By assumption (iii) of Proposition~\ref{prop:dominance}, Literati and DPDT employ the same CART-based discretization function $\Phi$. Provided the candidate schedule $B_d$ is the same for all depths, DPDT and Literati utilize identical collections of per-feature candidate sets $\{U_s^{(j)}\}_{j=1}^{M}$ for any dataset $\mathcal{D}_s \subseteq \mathcal{D}$. Under $K=1$, the constraint $|\mathbf{k}| \leq K$ forces the Literati refine action set $\mathcal{A}_{\mathrm{refine}}$ to be empty at every shape function state, leaving the commit action $a_{\mathrm{commit}}$ as the only available action. Each shape function state $(\mathcal{D}_s, d, j, \{u\})$ thus deterministically transitions to the data state pair $\bigl((\mathcal{D}_s(f_j^{\{u\}})_L, {d-1}),\allowbreak\, (\mathcal{D}_s(f_j^{\{u\}})_R, {d-1})\bigr)$ with cost zero, exactly mirroring the partition induced by the DPDT split $a_{j,u}$~\citep{kohler2025breiman}. The two methods therefore induce identical search spaces, yielding the same discretized feasible set $\mathcal{T}_{\mathrm{DPDT}} = \mathcal{T}_{\mathrm{Liter}}^{K=1}$. Combined with Equation~\ref{eq:stage1}, this gives $\mathcal{L}^*_{\mathrm{Liter}}\big|_{K=1} = \mathcal{L}^*_{\mathrm{DPDT}}$. For $K \geq 2$, the constraint $|\mathbf{k}| \leq K$ relaxes that of $K=1$, so $\mathcal{T}_{\mathrm{Liter}}^{K=1} \subseteq \mathcal{T}_{\mathrm{Liter}}^{K}$ and minimizing the Literati objective over the larger set gives $\mathcal{L}^*_{\mathrm{Liter}}\big|_{K \geq 2} \leq \mathcal{L}^*_{\mathrm{Liter}}\big|_{K=1} = \mathcal{L}^*_{\mathrm{DPDT}}$. \qed

Under the corresponding matched conditions, CART's regularized loss upper-bounds the optimal regularized loss of DPDT~\cite{kohler2025breiman} and, consequently, that of Literati.

\subsection{Growth Bound of our AND/OR Graph}
We analyze the size of the AND/OR graph defined in Section~\ref{sec:aograph} in terms of the number of training instances $N$, features $M$, maximum tree depth $D$, and shape function complexity budget $K$, by examining the branching factor at each state type.

Each non-terminal data state $(\mathcal{D}_s, d)$ has at most $\sum_{j=1}^{M}|\mathcal{U}_s^{(j)}| \leq MN$ split successors (one shape function state per feature--change-point pair) and one terminal successor. Each shape function state $(\mathcal{D}_s, d, j, \mathbf{k})$ has at most $|\mathcal{U}_s^{(j)}| - |\mathbf{k}| \leq N - 1$ refine successors (since the ordering constraint $u > \max(\mathbf{k})$ excludes $\max(\mathbf{k})$ and all values below it) and one commit action producing two AND-successors at depth budget $d-1$, for at most $N+1$ successor states. Since $|\mathbf{k}|$ ranges from $1$ to $K$, there are up to $K$ successive layers of shape function states between consecutive data state levels. The effective branching factor from a data state at depth budget $d$ to child data states at depth budget $d-1$ is therefore $O(MN^K)$, and applying this recursively across $D$ depth levels yields
\begin{gather}
    |\mathcal{S}_{\mathrm{data}}| = O\!\left((MN^K)^{D}\right).
\end{gather}
Each data state gives rise to $O(MN^K)$ shape function states, so the total state space is
\begin{gather}
    |\mathcal{S}| = O\!\left((MN^K)^{D+1}\right),
\end{gather}
dominated by shape function states. This exponential growth motivates the anytime property of Literati: a procedure that achieves the optimal or near-optimal solution early, without exhaustively expanding the graph, is critical for practical scalability.

\subsubsection{Adaptive Discretization}
Here, we analyze the growth rate of the AND/OR graph under adaptive discretization. Given a discretization schedule $B_d,\, d = 1, \ldots, D$, fitting a CART tree with at most $B_d+1$ leaves at each data state with depth budget $d$ yields at most $B_d$ candidate thresholds distributed across features, so $\sum_{j=1}^{M} |U_s^{(j)}| \leq B_d$. The branching factor at a data state reduces from $MN$ to at most $B_d$ split successors, and the refine branching at each shape function layer reduces from $N - 1$ to at most $B_d - 1$. With up to $K$ shape function layers, the effective branching from data states at depth budget $d$ to child data states at depth budget $d - 1$ is $O(B_d^K)$. Since this branching factor varies with $d$, the total state space is bounded by the product across depth levels:
\begin{gather}
    |\mathcal{S}| = {O\!\left(D \cdot \prod_{d=1}^{D} B_d^K \right)}.
\end{gather}
As $B_d \ll M \cdot N$, adaptive discretization significantly reduces the size of the AND/OR graph. 

\section{Literati Algorithmic Details}\label{sec:pseudocode}

In this section, we provide more details on our algorithm, including notable design decisions that may be useful for re-implementing our algorithm, as well as pseudocode in Algorithms~\ref{alg:literati-main}--\ref{alg:literati-backprop}. Source code for Literati can be found at \url{https://github.com/upadhyan/literati}.

\paragraph{Architecture.} Literati is implemented in Rust with Python bindings. Each state $s \in \mathcal{S}$ is represented as a struct storing the index vector identifying samples in $\mathcal{D}_s$, the heuristic values $V(s)$, $\widetilde{V}(s)$, and $\overline{V}(s)$, a COMPLETE flag, and additional fields described in the paragraphs below. The AND dynamics induced by a commit action are encoded in a separate struct that tracks the joint solved status of $(s_L, s_R) = F(s_{\text{shape}}, a_{\text{commit}})$ and supports round-robin selection (Section~\ref{sec:and_selection}).

\paragraph{Efficient Data Partitioning.} Applying a shape function $f_j^{\mathbf{k}}$ requires partitioning $\mathcal{D}_s$ according to $\mathbf{k}$. To avoid repeatedly scanning $\mathcal{D}_s$, when we expand a data state $(\mathcal{D}_s, d)$ we precompute, for every feature $j$, the sample indices falling in each interval between consecutive candidates in $U_s^{(j)}$, and store these bin partitions on the data state's struct. Descendant shape function states sharing the same ancestor data state hold pointers to these bins rather than copies. A shape function on feature $j$ can then be represented compactly as a boolean array over bins indicating left or right assignment, and partitioning $\mathcal{D}_s$ reduces to a single linear scan of this array followed by concatenation of the corresponding index vectors.

\paragraph{Policy Queues.} Each partial policy ($\pi$, $\widetilde{\pi}$, and $\overline{\pi}$) is represented as a min-queue over $\mathcal{A}(s)$ keyed on the corresponding $Q$-values, with the queues stored on the state's struct. The queues are re-sorted whenever a state is touched during backpropagation. For $\pi$ and $\overline{\pi}$, the prescribed action is simply the queue head. For $\widetilde{\pi}$, since $\widetilde{V}$ is potentially inadmissible, we instead traverse the queue until we find an action whose successors are not all COMPLETE (Section~\ref{sec:or_selection}).

\paragraph{Eager Initialization of Commit Successors.} The heuristic values $V$, $\widetilde{V}$, and $\overline{V}$ at a shape function state $s_{\text{shape}} = (\mathcal{D}_s, d, j, \mathbf{k})$ are functions of the two data states resulting from its commit action. We therefore eagerly construct and populate the structs for $(\mathcal{D}_s(f_j^{\mathbf{k}})_L, d-1)$ and $(\mathcal{D}_s(f_j^{\mathbf{k}})_R, d-1)$ (using the bin pointers inherited from the ancestor data state) when $s_{\text{shape}}$ is added to the graph, and retain them until $s_{\text{shape}}$ is itself expanded. This avoids recomputing partition statistics that are reused throughout backpropagation, yielding substantial runtime savings in practice.

\paragraph{Lazy Backpropagation.}
To avoid unnecessary updates and priority queue reorderings, backpropagation along an ancestor chain is terminated as soon as the quantities $V(s)$, $\widetilde{V}(s)$, $\overline{V}(s)$, and the \textsc{Complete} flag are all unchanged after updating state $s$. Propagation continues to the parent of $s$ only if at least one of these quantities is modified.

\paragraph{Regularization Reparameterization.} For analytical clarity, the objective in Equation~\ref{eq:objective} charges $\lambda + \alpha\, \mathbf{K}(f_v)$ at every internal node. However, to make the two regularization terms easier for users to interpret independently, our implementation instead exposes parameters $\lambda_r \geq 0$ and $\alpha_r \geq 0$ and charges $\lambda_r + \alpha_r\,(\mathbf{K}(f_v) - 1)$ per node, so that $\alpha_r$ penalizes only the structure beyond a standard threshold split. This is an affine reparameterization of the original objective, recovered exactly by setting $\lambda = \lambda_r - \alpha_r$ and $\alpha = \alpha_r$. In the AND/OR graph, the split transition cost decomposes as $c(s, a_{j,u}) = (\lambda_r - \alpha_r) + \alpha_r = \lambda_r$ and the refine cost remains $c(s, a_{j,\mathbf{k},u}) = \alpha_r$, both non-negative, so the correctness of the AND/OR formulation, the AO* search procedure, and all theoretical results are preserved.

\paragraph{Caching} Unlike Branches~\cite{chaouki2025branches}, we do not cache states keyed on their reaching subset $\mathcal{D}_s$. Empirically, the memory cost of maintaining such a cache, together with the lookup overhead at every state construction, outweighs any savings from reuse. Two factors compound this: (i) our adaptive candidate generation produces per-node threshold sets $U_s^{(j)}$ that yield largely distinct partitions across the search tree, reducing cache hit rates, and (ii) caching invalidates the memory advantage Literati gains from pruning states, as we would still need to retain pruned states for future reuse. Instead, we maintain a limited cache of $V$, $\bar{V}$, and $\widetilde{V}$ values for the 10,000 most recently explored data states. This is substantially more memory efficient because it avoids retaining pointers to the underlying states, allowing Literati to free pruned states from memory while still benefiting from prior exploration through tighter bounds.

\subsection{Pseudocode}
\begin{algorithm}[H]
\DontPrintSemicolon
\caption{Literati: main loop}\label{alg:literati-main}
\resizebox{0.94\textwidth}{!}{%
\begin{minipage}{\linewidth}
\KwIn{data $\mathcal{D}$, depth budget $D$; complexity budget $K$; regularizers $\lambda, \alpha$; time limit $\tau$}
\KwOut{Shape Generalized Tree}
\SetKwFunction{InitDataState}{InitDataState}
\SetKwFunction{Select}{Select}
\SetKwFunction{Expand}{Expand}
\SetKwFunction{Backpropagate}{Backpropagate}
\SetKwFunction{ExtractTree}{ExtractTree}
\SetKwFunction{TimedOut}{TimedOut}
$s_0 \gets (\mathcal{D}, D)$ \tcp*{root data state}
\InitDataState{$s_0$}\;
\While{$\neg\,\mathrm{COMPLETE}(s_0) \;\wedge\; \neg\,\TimedOut()$}{
  $s \gets \Select(s_0)$\;
  \Expand{$s$}\;
  \Backpropagate{$s$}\;
}
\Return{$\ExtractTree(s_0)$} \tcp*{unroll $\bar\pi$ from the root}
\end{minipage}%
}
\end{algorithm}

\begin{algorithm}[H]
\DontPrintSemicolon
\caption{Selection}\label{alg:literati-select}
\resizebox{0.94\textwidth}{!}{%
\begin{minipage}{\linewidth}
\SetKwProg{Fn}{Function}{:}{}
\SetKwFunction{Select}{Select}
\SetKwFunction{ANDSelection}{ANDSelection}

\Fn{\Select{$s$}}{
  \lIf{$\neg\,\mathrm{expanded}(s)$}{\Return{$s$}}
  $a \gets \tilde\pi(s)$\;
  \uIf{$|F(s, a)| > 1$}{
    \Return{$\Select(\ANDSelection(F(s, a),\, s))$}\;
  }\Else{
    $\{s'\} \gets F(s, a)$\;
    \Return{$\Select(s')$}\;
  }
}

\BlankLine
\Fn{\ANDSelection{$\{s_L, s_R\},\, s$}}{
  \lIf{$\mathrm{COMPLETE}(s_L)$}{\Return{$s_R$}}
  \lIf{$\mathrm{COMPLETE}(s_R)$}{\Return{$s_L$}}
  \uIf{$\mathrm{chooseLeft}(s)$}{
    $\mathrm{chooseLeft}(s) \gets \textsf{false}$\;
    \Return{$s_L$}\;
  }\Else{
    $\mathrm{chooseLeft}(s) \gets \textsf{true}$\;
    \Return{$s_R$}\;
  }
}
\end{minipage}
}
\end{algorithm}

\begin{algorithm}[H]
\DontPrintSemicolon
\caption{Expansion}\label{alg:literati-expand}
\resizebox{0.88\textwidth}{!}{%
\begin{minipage}{\linewidth}
\SetKwProg{Fn}{Function}{:}{}
\SetKwFunction{Expand}{Expand}

\SetKwFunction{Discretize}{Discretize}
\SetKwFunction{Partition}{Partition}
\SetKwFunction{RetrieveChangepoints}{RetrieveChangepoints}
\SetKwFunction{InitSFState}{InitSFState}
\SetKwFunction{InitDataState}{InitDataState}
\SetKwFunction{FitCART}{FitCART}

\Fn{\Expand{$s$}}{
  \uIf(\tcp*[f]{data state}){$s$\text{ is Data State}}{
    $\mathcal{D}_s,d \gets s$\;

    $t \gets \overline{(\mathcal{D}_s,d)}$\;
    $V(t) \gets \mathrm{Err}(\mathcal{D}_s)$;\quad $\overline{V}(t) \gets V(t)$;\quad $\widetilde{V}(t) \gets V(t)$\;
    $\mathrm{COMPLETE}(t) \gets \textsf{true}$\;
    add $\bar a$ with cost $0$ and $F(s,\bar a)=\{t\}$ to $\mathcal{A}(s)$\;
    $U_s^{(1)},\ldots,U_s^{(M)} \gets \Discretize(\mathcal{D}_s,d)$\;
    \ForEach{$j=1,\ldots,M$}{ 
        \tcp{add Initial Shape Function states}
        \ForEach{$u \in U_s^{(j)}$}{
            $s' \gets (\mathcal{D}_s, d, j, \{u\})$\;
            $\InitSFState(s')$\;
            add $a_{j,u}$ with cost $\lambda + \alpha$,\, $F = \{s'\}$ to $\mathcal{A}(s)$\;
        }
    }
    
  }\Else(\tcp*[f]{shape function state}){
    
    \ForEach{$s' \in \{s_L, s_R\}$}{
      $\InitDataState(s')$\;
    }
    add $a_{\text{commit}}$ with cost $0$,\, $F = \{s_L, s_R\}$ to $\mathcal{A}(s)$\;
    \If{$|\mathbf{k}| < K$}{
      $U_s^{(j)}\gets \RetrieveChangepoints(\mathcal{D}_s,j)$\;
      \ForEach{$u \in U_{s}^{(j)}$ \KwSty{with} $u > \max(\mathbf{k})$}{
        $s' \gets (\mathcal{D}_s, d, j, \mathbf{k} \cup \{u\})$\;
        $\InitSFState(s')$\;
        add $a_{j, \mathbf{k}, u}$ with cost $\alpha$,\, $F = \{s'\}$ to $\mathcal{A}(s)$\;
      }
    }
  }
  $\mathrm{expanded}(s) \gets \textsf{true}$\;
}
\Fn{\InitSFState{$s$}}{
    $(\mathcal{D}_s, d, j, \mathbf{k}) \gets s$\;
    $(\mathcal{D}_L, \mathcal{D}_R) \gets \Partition(\mathcal{D}_s,\, f_j^{\mathbf{k}})$\;
    $s_L \gets (\mathcal{D}_L, d - 1)$;\quad $s_R \gets (\mathcal{D}_R, d - 1)$\;
    \tcp{Precompute left and right Data State statistics.}
    $V(s_L) \gets \min(\text{Err}(\mathcal{D}_L), \lambda + \alpha)$; $V(s_R) \gets \min(\text{Err}(\mathcal{D}_R), \lambda + \alpha)$\;
    $\overline{V}(s_L) \gets \text{Err}(\mathcal{D}_L); \overline{V}(s_R) \gets \text{Err}(\mathcal{D}_R)$\;
    \uIf{$d>1$}{
        $T_L \gets \FitCART(\mathcal{D}_L,d-1)$;\quad $T_R \gets \FitCART(\mathcal{D}_R,d-1)$\;
        $\widetilde{V}(s_L) \gets \mathcal{L}(T_L,\mathcal{D}_L)$;\quad $\widetilde{V}(s_R) \gets \mathcal{L}(T_R,\mathcal{D}_R)$\;
    }\Else{
        $\widetilde{V}(s_L) \gets \mathrm{Err}(\mathcal{D}_L)$;\quad $\widetilde{V}(s_R) \gets \mathrm{Err}(\mathcal{D}_R)$\;
    }
    \tcp{$\mathcal{L}(T,\mathcal{D})$ is the objective value of Equation~\ref{eq:objective}}
    \ForEach{$s' \in \{s_L,s_R\}$}{
        \InitDataState{$s'$}\;
    }
    \tcp{Assign Shape Function State statistics}
    $V(s) \gets \min(\alpha, V(s_L) + V(s_R))$\;
    $\overline{V}(s) \gets  \overline{V}(s_L) + \overline{V}(s_R)$\;
    $\widetilde{V}(s) \gets \widetilde{V}(s_L) + \widetilde{V}(s_R)$\;
    $\mathrm{COMPLETE}(s) \gets \mathrm{FALSE}$
}

\Fn{\InitDataState{$s$}}{
    $(\mathcal{D}_s,d) \gets s$\;
    $\mathrm{COMPLETE}(s) \gets (d=0) \vee (\mathrm{Err}(\mathcal{D}_s) \leq \lambda + \alpha)$\;
    \If{$\mathrm{COMPLETE}(s)$}{
        $V(s) \gets \mathrm{Err}(\mathcal{D}_s)$\;
        $\overline{V}(s) \gets V(s)$;\quad $\widetilde{V}(s) \gets V(s)$\;
        \Return{}\;
    }
    \If{$V(s), \overline{V}(s), \widetilde{V}(s)\ \mathrm{do\ not\ exist}$}{
        $V(s) \gets \min(\mathrm{Err}(\mathcal{D}_s), \lambda + \alpha)$\;
        $\overline{V}(s) \gets \mathrm{Err}(\mathcal{D}_s)$\;
        $T \gets \FitCART(\mathcal{D}_s,d)$\;
        $\widetilde{V}(s) \gets \mathcal{L}(T,\mathcal{D}_s)$\;
    }
}
\end{minipage}
}
\end{algorithm}

\begin{algorithm}[H]
\DontPrintSemicolon
\caption{Backpropagation}\label{alg:literati-backprop}
\resizebox{0.94\textwidth}{!}{%
\begin{minipage}{\linewidth}
\SetKwProg{Fn}{Function}{:}{}
\SetKwFunction{Backpropagate}{Backpropagate}

\Fn{\Backpropagate{$s$}}{
  $\textsc{old} \gets (V(s),\, \bar{V}(s),\, \tilde{V}(s),\, \mathrm{COMPLETE}(s))$\;

  \tcp{primal}
  $\bar\pi(s) \gets \arg\min_{a \in \mathcal{A}(s)} \bar{Q}(s, a)$;\quad
  $\bar{V}(s) \gets \bar{Q}(s, \bar\pi(s))$\;

  \tcp{prune}
  $\mathcal{A}(s) \gets \{\, a \in \mathcal{A}(s) \,:\, Q(s, a) \leq \bar{V}(s) \,\}$\;
  Retain the incumbent policy and its solution graph before freeing pruned subgraphs\;
  Remove subgraphs of any pruned actions\;
  \tcp{Policy Updates}
  $\pi(s) \gets \arg\min_{a \in \mathcal{A}(s)} Q(s, a)$;\quad
  $V(s) \gets Q(s, \pi(s))$\;

  \tcp{COMPLETE check}
  \If{$\forall\, s' \in F(s, \pi(s)) \,:\, \mathrm{COMPLETE}(s')$}{
    $\mathrm{COMPLETE}(s) \gets \mathrm{TRUE}$\;
    $\mathcal{A}(s) \gets \{\pi(s)\}$\; \tcp{Search complete; Fix policy}
  }
  \tcp{Select an action with unsolved successors}
  \uIf{$\mathrm{COMPLETE}(s)$}{
    $\tilde\pi(s) \gets \pi(s)$;\quad $\tilde V(s) \gets V(s)$\;
  }\Else{
    $\mathcal{A}_{\mathrm{open}}(s) \gets \{a\in\mathcal{A}(s):\exists s'\in F(s,a),\ \neg\mathrm{COMPLETE}(s')\}$\;
    $\tilde\pi(s) \gets \arg\min_{a\in\mathcal{A}_{\mathrm{open}}(s)}\tilde Q(s,a)$\;
    $\tilde V(s) \gets \tilde Q(s,\tilde\pi(s))$\;
  }
  \tcp{Propagate Upward}
  \If{$(V(s), \bar{V}(s), \tilde{V}(s), \mathrm{COMPLETE}(s)) \ne \textsc{old} \;\wedge\; \mathrm{parent}(s)\ \mathrm{exists}$}{
    \Backpropagate{$\mathrm{parent}(s)$}\;
  }
}
\end{minipage}
}
\end{algorithm}

\begin{algorithm}[H]
\DontPrintSemicolon
\caption{Discretization with caching}\label{alg:literati-discretize}
\resizebox{0.94\textwidth}{!}{%
\begin{minipage}{\linewidth}
\SetKwProg{Fn}{Function}{:}{}
\SetKwFunction{Discretize}{Discretize}
\SetKwFunction{FitCART}{FitCART}
\Fn{\Discretize{$\mathcal{D},d$}}{
  \lIf{$(\mathcal{D},d) \in \mathrm{cache}$}{\Return{$\mathrm{cache}[(\mathcal{D},d)]$}}
  $T \gets \FitCART(\mathcal{D},\,\mathrm{maxLeaves}=B_d+1)$\;
  $U_{\mathcal{D}}^{(j)} \gets \emptyset \quad \forall\, j \in \{1, \ldots, M\}$\;
  \ForEach{internal node $v \in \mathcal{I}(T)$}{
    let $(j_v, u_v)$ be the feature and threshold at node $v$\;
    $U_{\mathcal{D}}^{(j_v)} \gets U_{\mathcal{D}}^{(j_v)} \cup \{u_v\}$\;
  }
  $\mathrm{cache}[(\mathcal{D},d)] \gets \{U_{\mathcal{D}}^{(j)}\}_{j=1}^{M}$\;
  \Return{$\{U_{\mathcal{D}}^{(j)}\}_{j=1}^{M}$}\;
}
\end{minipage}
}
\end{algorithm}
\newpage
\section{Additional Experimental Details} \label{sec:add_exp_details}
All experiments were run on Intel Granite Forest compute nodes (Xeon 6972P @ 2.4 GHz, 12 vCPUs, 64 GB RAM). For the Optimization Performance experiments (Section~\ref{sec:optimization_performance}), all regularization is turned off for all approaches.  

\paragraph{Baseline Implementation Details:} For CART, we utilize the implementation found in Scikit-Learn~\cite{scikit-learn}. For DPDT, we utilize the official implementation provided by the original authors~\cite{kohler2025breiman} found at \url{https://github.com/KohlerHECTOR/DPDTreeEstimator}. For STreeD~\cite{van2023necessary}, we utilize the official implementation provided by the original authors found at \url{https://github.com/AlgTUDelft/pystreed}. For ConTree~\cite{brita2025optimal}, we utilize the official implementation provided by the original authors found at \url{https://github.com/ConSol-Lab/contree/tree/main}. For Branches~\cite{chaouki2025branches}, we utilize the official implementation provided by the original authors found at \url{https://github.com/Chaoukia/branches}. NOTE: when faced with categorical variables, Branches supports multi-way branching, routing each category to a separate sub-tree. For fair comparison with binary trees, we turn this feature off in our experiments by setting the \texttt{encoding} parameter to \texttt{"multi"} or \texttt{"binary"}\footnote{See Line 27 of \url{https://github.com/Chaoukia/branches/blob/main/src/branches.py}}.  For ShapeCART and ShapeTAO~\cite{upadhyaempowering}, we utilize the official implementation provided by the original authors found at \url{https://github.com/optimal-uoft/Empowering-DTs-via-Shape-Functions}.

\subsection{Generalization Experiment Details}\label{sec:hp}
Below are the hyperparameters tuned for each evaluated approach. Depth is treated as a hyperparameter for all models and tuned over $D \in \{2,\ldots,6\}$ following \citet{upadhyaempowering}.
\begin{itemize}\raggedright
    \item \textbf{Literati}: \texttt{impurity}~$\in \{$gini, entropy$\}$; \texttt{lambda\_}, \texttt{alpha\_}~$\in \{1,2,4,8,16,32,64\}$ with \texttt{alpha\_}~$\leq$~\texttt{lambda\_}; \texttt{max\_order}~$\in \{2,3\}$; \texttt{cart\_nodes\_list}~$\in \{(9,9),(5,5),(17,17),(17,9),(9,5)\}$\footnote{$B_d = \texttt{cart\_nodes\_list}[D-d] - 1$. For both Literati and DPDT, tuple entries specify CART leaf budgets from the root downward. If $D-d$ is out of index for the tuple, we set $B_d = K$ to allow for shape functions at all depths}. \texttt{lambda\_}, \texttt{alpha\_} divided by \texttt{n\_samples}.
    \item \textbf{ShapeCART}: \texttt{criterion}~$\in \{$gini, entropy$\}$; \texttt{min\_samples\_split}~$\in \{2,4,8,16,32,64\}$; \texttt{min\_samples\_leaf}~$\in \{1,2,4,8,16,32\}$; \texttt{min\_impurity\_decrease}~$\in \{0, 10^{-4}, 5{\cdot}10^{-4}, 10^{-3}, 5{\cdot}10^{-3}, 10^{-2}\}$; \texttt{inner\_max\_leaf\_nodes}~$\in \{3,4\}$. Note that this parameter controls the max shape function complexity (e.g.\ $K = \texttt{inner\_max\_leaf\_nodes} - 1$); \texttt{inner\_min\_samples\_leaf}~$\in \{1, 10^{-4}, 5{\cdot}10^{-4}, 10^{-3}, 5{\cdot}10^{-3}, 10^{-2}\}$. 
    \item \textbf{ShapeTAO}: Same as ShapeCART \textbf{plus} \texttt{tao\_reg}~$\in \{1,2,4,8,16,32,64,128\}$. \texttt{lambda\_}~$=$~\texttt{tao\_reg} $/$ \texttt{n\_samples}.
    \item \textbf{CART}: \texttt{criterion}~$\in \{$gini, entropy$\}$; \texttt{min\_samples\_split}~$= 2^k$, $k \in \{1,\dots,5\}$; \texttt{min\_samples\_leaf}~$\in \{1,\dots,32\}$; \texttt{min\_impurity\_decrease}~$\in \{0, 10^{-4}, 5{\cdot}10^{-4}, 10^{-3}, 5{\cdot}10^{-3}, 10^{-2}\}$; \texttt{ccp\_alpha}~$\in \{0, 10^{-4}, 5{\cdot}10^{-4}, 10^{-3}, 5{\cdot}10^{-3}, 10^{-2}\}$. 
    \item \textbf{AxTAO}: \texttt{criterion}~$\in \{$gini, entropy$\}$; \texttt{min\_samples\_split}~$= 2^k$, $k \in \{1,\dots,5\}$; \texttt{min\_samples\_leaf}~$\in \{1,\dots,32\}$; \texttt{min\_impurity\_decrease}~$\in \{0, 10^{-4}, 5{\cdot}10^{-4}, 10^{-3}, 5{\cdot}10^{-3}, 10^{-2}\}$; \texttt{lambda\_}~$\in \{0, 10^{-4}, 5{\cdot}10^{-4}, 10^{-3}, 5{\cdot}10^{-3}, 10^{-2}\}$. \texttt{lambda\_} used as raw fraction (not scaled by n).
    \item \textbf{DPDT}: \texttt{criterion}~$\in \{$gini, entropy$\}$; \texttt{min\_samples\_split}~$= 2^k$, $k \in \{1,\dots,5\}$; \texttt{min\_samples\_leaf}~$\in \{1,\dots,32\}$; \texttt{min\_impurity\_decrease}~$\in \{0, 10^{-4}, 5{\cdot}10^{-4}, 10^{-3}, 5{\cdot}10^{-3}, 10^{-2}\}$; \texttt{cart\_nodes\_list}~$\in \{(9,9),(5,5),(17,17),(17,9),(9,5)\}$. 
    \item \textbf{Branches}: \texttt{lambda\_}~$\in \{1,2,4,8,16,32,64\}$; \texttt{n\_quantiles}~$\in [4,16]$. QuantileBinarizer preprocessing; \texttt{lambda\_} scaled by \texttt{n\_samples}.
    \item \textbf{ConTree}: \texttt{lambda\_}~$\in \{1,2,4,8,16,32,64\}$. \texttt{lambda\_} mapped to \texttt{complexity\_cost}~$=$~\texttt{lambda\_} $/$ \texttt{n\_samples}.
    \item \textbf{STreeD}: \texttt{lambda\_}~$\in \{1,2,4,8,16,32,64\}$; \texttt{n\_quantiles}~$\in [4,16]$. QuantileBinarizer preprocessing; \texttt{lambda\_} mapped to \texttt{cost\_complexity}~$=$~\texttt{lambda\_} $/$ \texttt{n\_samples}.
    \item \textbf{LDS-DL8.5}: \texttt{lambda\_}~$\in \{1,2,4,8,16,32,64\}$; \texttt{n\_quantiles}~$\in [4,16]$. QuantileBinarizer preprocessing; \texttt{lambda\_} mapped to \texttt{min\_support}
    \item \textbf{CADL8.5}: \texttt{lambda\_}~$\in \{1,2,4,8,16,32,64\}$; \texttt{n\_quantiles}~$\in [4,16]$. QuantileBinarizer preprocessing; \texttt{lambda\_} mapped to \texttt{min\_support}
\end{itemize}

\begin{table}[!ht]
\caption{Dataset statistics and source information. }
\centering
\label{tab:dataset_info}
\begin{tabular}{@{}cccccc@{}}
\toprule
Dataset       & $N$    & $M$  & $|\mathcal{C}|$ & Source & Source ID \\ \midrule
adult         & 45222  & 108  & 2               & OpenML & 179       \\
avila         & 20867  & 10   & 12              & OpenML & 42932     \\
bank          & 1372   & 4    & 2               & OpenML & 1462      \\
bean          & 13611  & 16   & 7               & UCI    & 602       \\
bidding       & 6321   & 9    & 2               & OpenML & 42889     \\
electricity   & 45312  & 13   & 2               & OpenML & 151       \\
eucalyptus    & 736    & 1487 & 5               & OpenML & 43924     \\
eye-movements & 10936  & 27   & 3               & OpenML & 1044      \\
eye-state     & 14980  & 14   & 2               & OpenML & 1471      \\
fault         & 1941   & 27   & 7               & OpenML & 40982     \\
gas-drift     & 13910  & 128  & 6               & OpenML & 1476      \\
htru          & 17898  & 8    & 2               & OpenML & 45558     \\
magic         & 13376  & 10   & 2               & OpenML & 44125     \\
mini-boone    & 130064 & 50   & 2               & OpenML & 41150     \\
mushroom      & 8124   & 94   & 2               & UCI    & 73        \\
occupancy     & 20560  & 5    & 2               & UCI    & 357       \\
page          & 5473   & 10   & 5               & OpenML & 30        \\
pendigits     & 10992  & 16   & 10              & OpenML & 32        \\
raisin        & 900    & 7    & 2               & UCI    & 850       \\
rice          & 3810   & 7    & 2               & UCI    & 545       \\
room          & 10129  & 16   & 4               & UCI    & 864       \\
segment       & 2310   & 18   & 7               & OpenML & 40984     \\
skin          & 245057 & 3    & 2               & OpenML & 1502      \\
wilt          & 4839   & 5    & 2               & OpenML & 40983     \\ \bottomrule
\end{tabular}
\end{table}
\newpage
\section{Selection Strategy Ablation Continued}\label{sec:add_inad_results}

\begin{figure}[!ht]
    \centering
    \includegraphics[width=\linewidth]{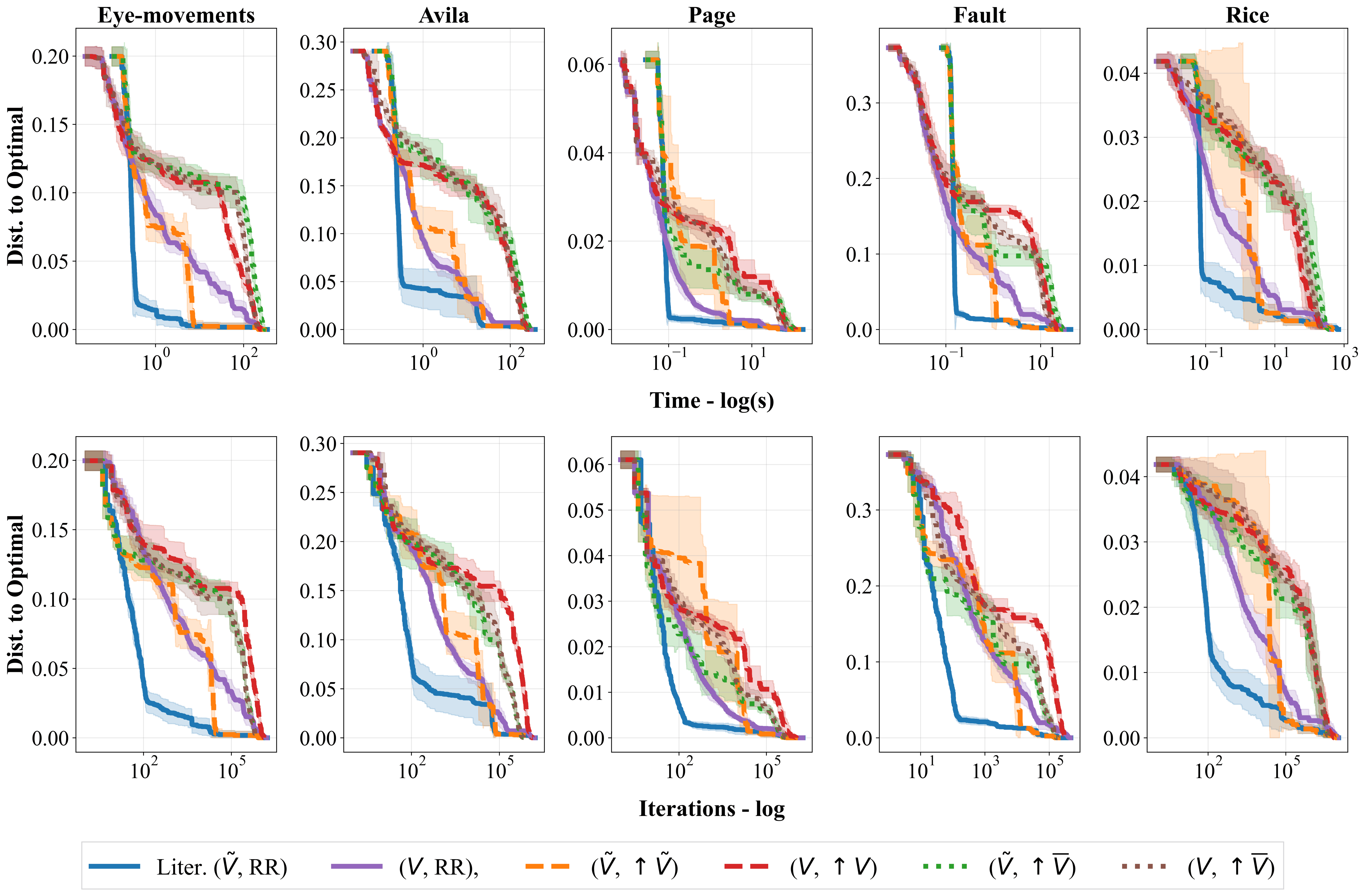}
    \caption{Distance to Optimal Solution across time (top) and iterations (bottom) for Literati (Full) and its ablations on Eye-Movements, Avila, Page, Fault, and Rice.}
    \label{fig:ablation_all}
\end{figure}

Figure~\ref{fig:ablation_all} shows results for Eye-Movements, Avila, Page, Fault, and Rice for both time and iterations. We observe that, across five datasets, using an informative heuristic for OR selection and round-robin AND selection configuration consistently achieves the best anytime performance among all configurations when looking both at time and iterations.

\newpage
\section{Other Inadmissible Heuristics}\label{sec:inadmissible_comp}

\begin{figure}[!ht]
    \centering
    \includegraphics[width=\linewidth]{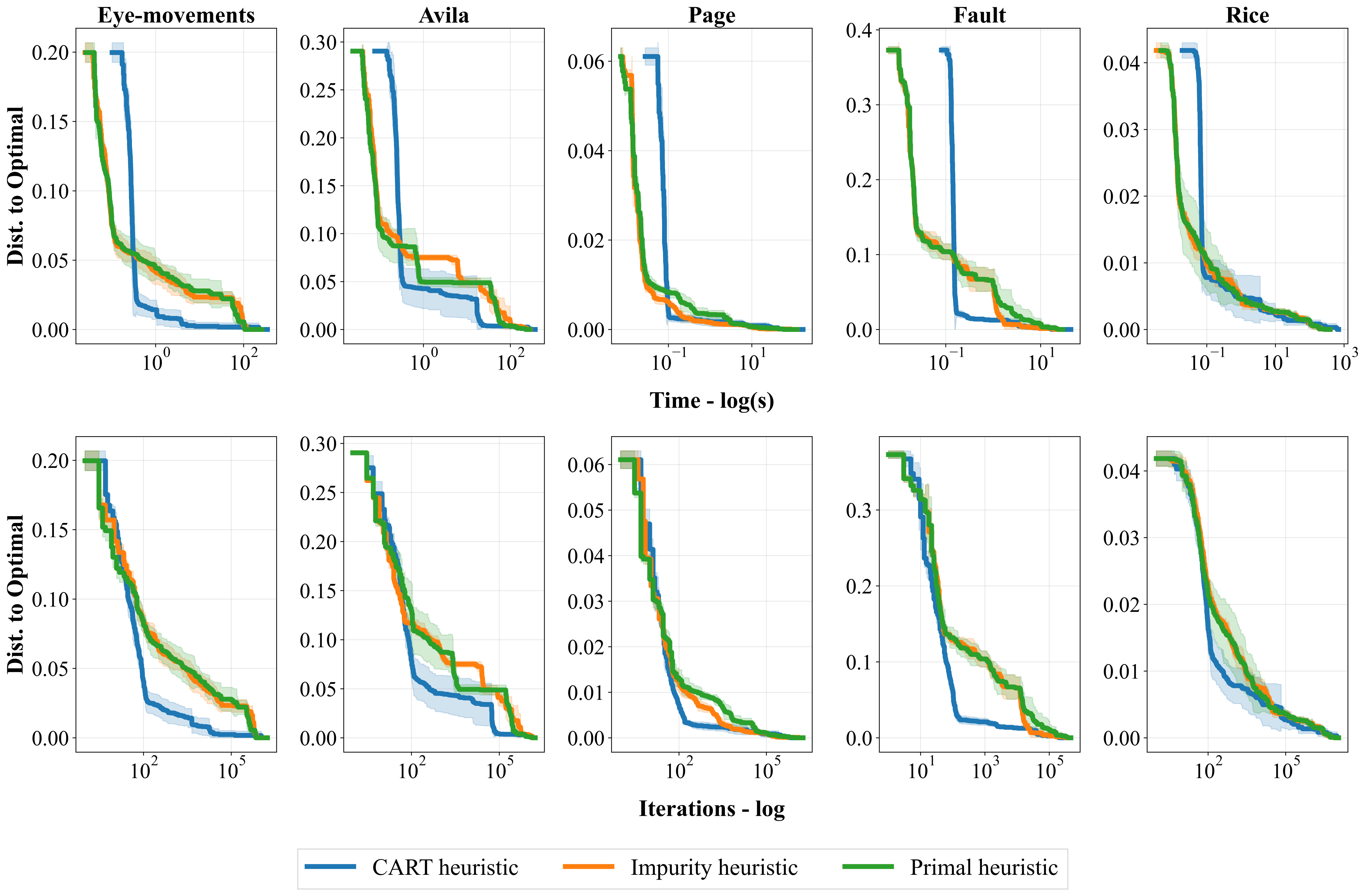}
    \caption{Distance to Optimal Solution across time (top) and iterations (bottom) for Literati under different $\widetilde{V}$ candidates. }
    \label{fig:inadmissible_heuristics}
\end{figure}
In this section, we evaluate two other candidates for inadmissible heuristics: 
\begin{itemize}
    \item \textbf{Primal:} Set $\widetilde{V} = \overline{V}$. This requires no additional computation, as the primal bound is already maintained during search.
    \item \textbf{Impurity:} Set $\widetilde{V}(\mathcal{D}_v, d) = I(\mathcal{D}_v) \cdot |\mathcal{D}_v|$, where $I(\cdot)$ denotes a weighted impurity measure (Gini or entropy). This is cheap to compute and better accounts for the class distribution. 
\end{itemize}
We follow the experimental setting used in our selection strategy analysis (Section~\ref{sec:anytime_ao_ablation} and Appendix~\ref{sec:add_inad_results}), replacing only the inadmissible heuristic driving the search. Figure~\ref{fig:inadmissible_heuristics} reports the distance-to-optimal against both wall-clock time and iteration count on Eye-Movements, Avila, Page, Fault, and Rice. Although the CART heuristic incurs minor additional overhead (in the order of $10^{-2}$--$10^{-1}$ seconds) during the initial iterations, it yields substantial runtime improvement on most datasets. This advantage is even more pronounced in the iteration view, where CART consistently approaches the optimum in fewer expansions than either the primal or impurity heuristics. The primal and impurity heuristics exhibit nearly similar behaviour across all five datasets, which we attribute to their shared myopic structure: both score a candidate split using only information local to the node. The CART heuristic instead performs an implicit lookahead by fitting a surrogate subtree, yielding tighter estimates.

\newpage
\section{Impact of Shapes}
\begin{figure}[!ht]
    \centering
    \includegraphics[width=\linewidth]{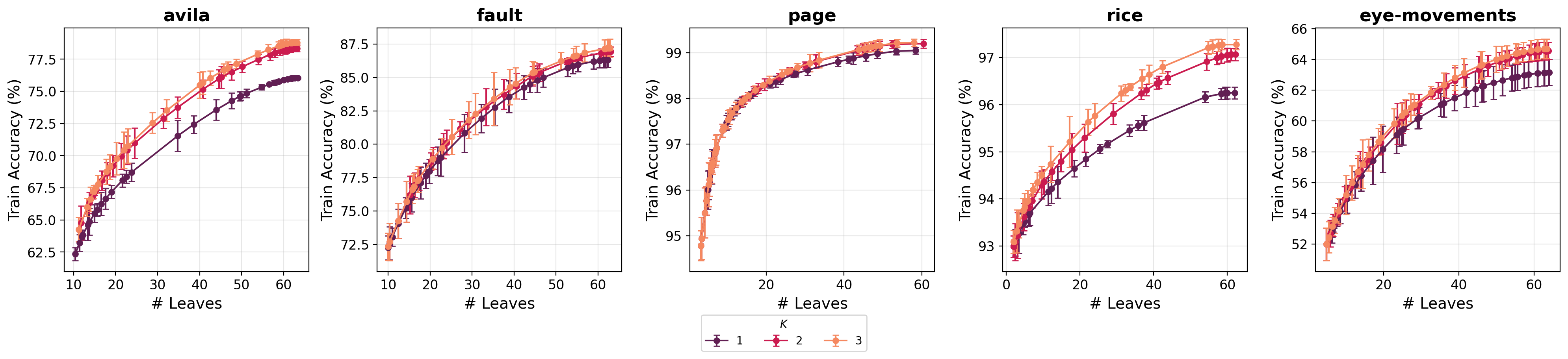}
    \caption{Train Accuracy across Number of Leaves. The horizontal bars describe the standard deviation across folds.}
    \label{fig:paretto}
\end{figure}
Decision trees are prized for their inherent interpretability. However, this interpretability is highly sensitive to the structural complexity of the tree, particularly the depth of decision paths and the number of leaves~\cite{molnar2020interpretable,luvstrek2016makes,SERDT,bertsimas2017optimal}. As an SGT, Literati leverages shape function branching, which yields higher-performing trees than axis-aligned threshold trees when controlling for the number of leaves. We empirically evaluate this benefit by comparing the train accuracy vs. model size tradeoff across shape complexity budgets. 

\paragraph{Setup.} We sweep $100$ values of $\lambda$ log-uniformly spaced over $[10^{-8}, 1]$ and train Literati with a maximum depth of $6$ using the light configuration from Appendix~\ref{sec:light_config}, imposing a one-hour time limit per run. We compare $K \in \{1, 2, 3\}$; under this configuration, where $K=1$ yields a threshold tree equivalent to DPDT~\cite{kohler2025breiman} (Proposition~\ref{prop:dominance}). Figure~\ref{fig:paretto} reports training accuracy against the number of leaves on five representative datasets: \emph{Avila}, \emph{Fault}, \emph{Page}, \emph{Rice}, and \emph{Eye-Movements}.

\paragraph{Results.} Across all five datasets, increasing $K$ results in higher training accuracy consistently across tree size. The effect is most pronounced on \emph{Avila} and \emph{Rice}. On \emph{Avila}, $K=2$ matches the training accuracy that $K=1$ reaches with $35$ leaves using only $25$; on \emph{Rice}, $K=2$ matches the $40$-leaf $K=1$ model with just $20$ leaves. We further observe diminishing returns in $K$: the gap between $K=2$ and $K=3$ is consistently smaller than the gap between $K=1$ and $K=2$. 

\section{Optimality Experiment Results Continued}\label{sec:opt_full_results}

\begin{figure}[htbp]
     \centering
     \begin{subfigure}[b]{0.45\textwidth}
         \centering
         \includegraphics[width=\textwidth]{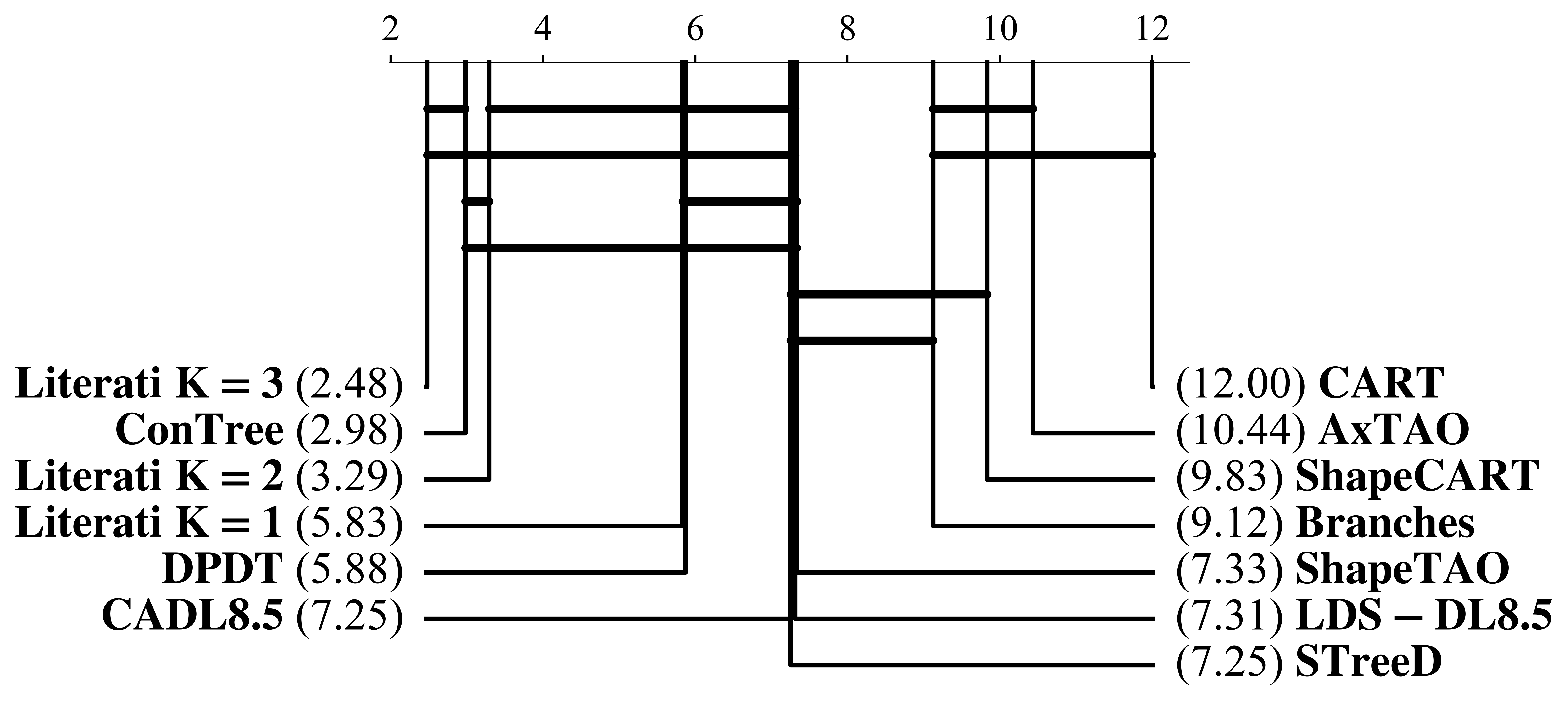}
         \caption{$D=3$}
         \label{fig:d3_trainacc}
     \end{subfigure}
     \hfill
     \begin{subfigure}[b]{0.45\textwidth}
         \centering
         \includegraphics[width=\textwidth]{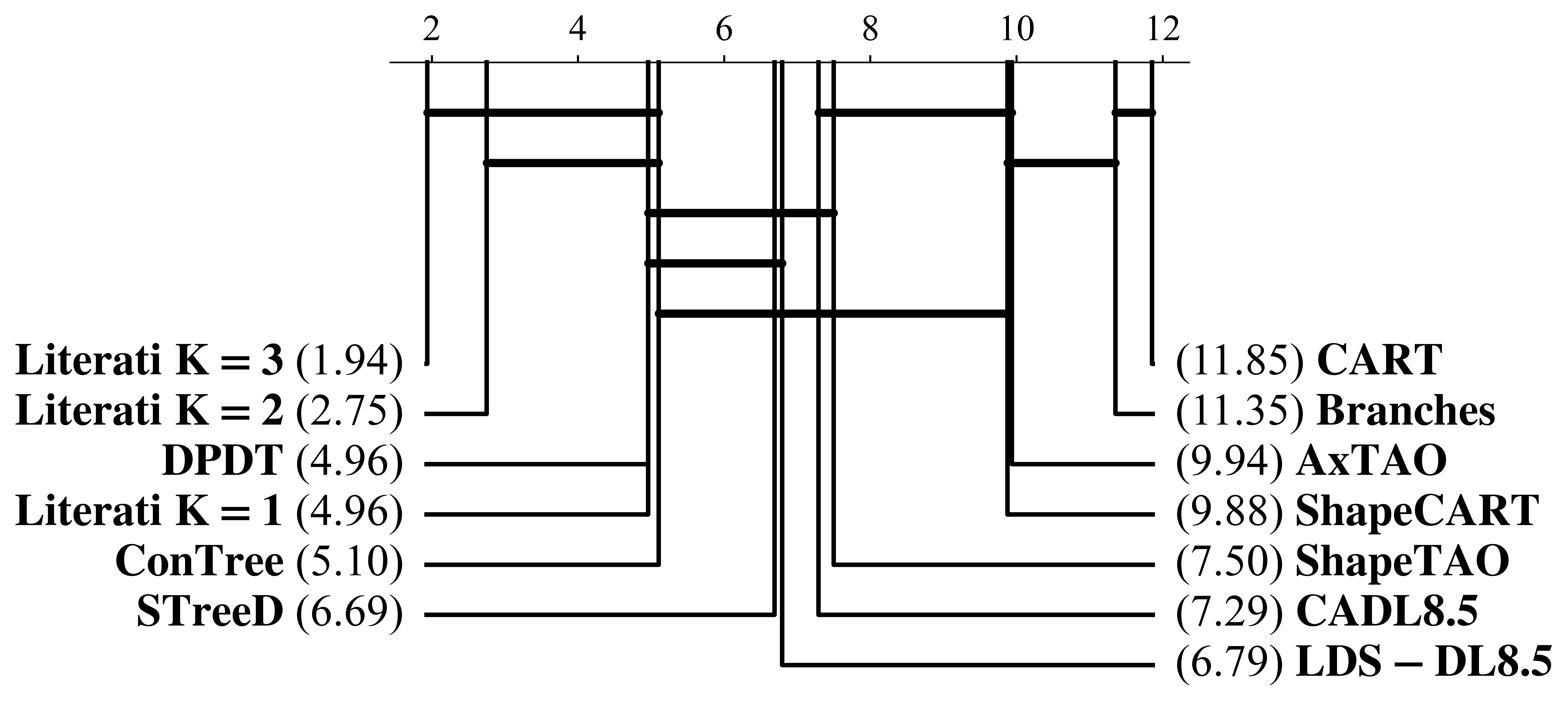}
         \caption{$D=4$}
         \label{fig:d4_trainacc}
     \end{subfigure}

     \begin{subfigure}[b]{0.45\textwidth}
         \centering
         \includegraphics[width=\textwidth]{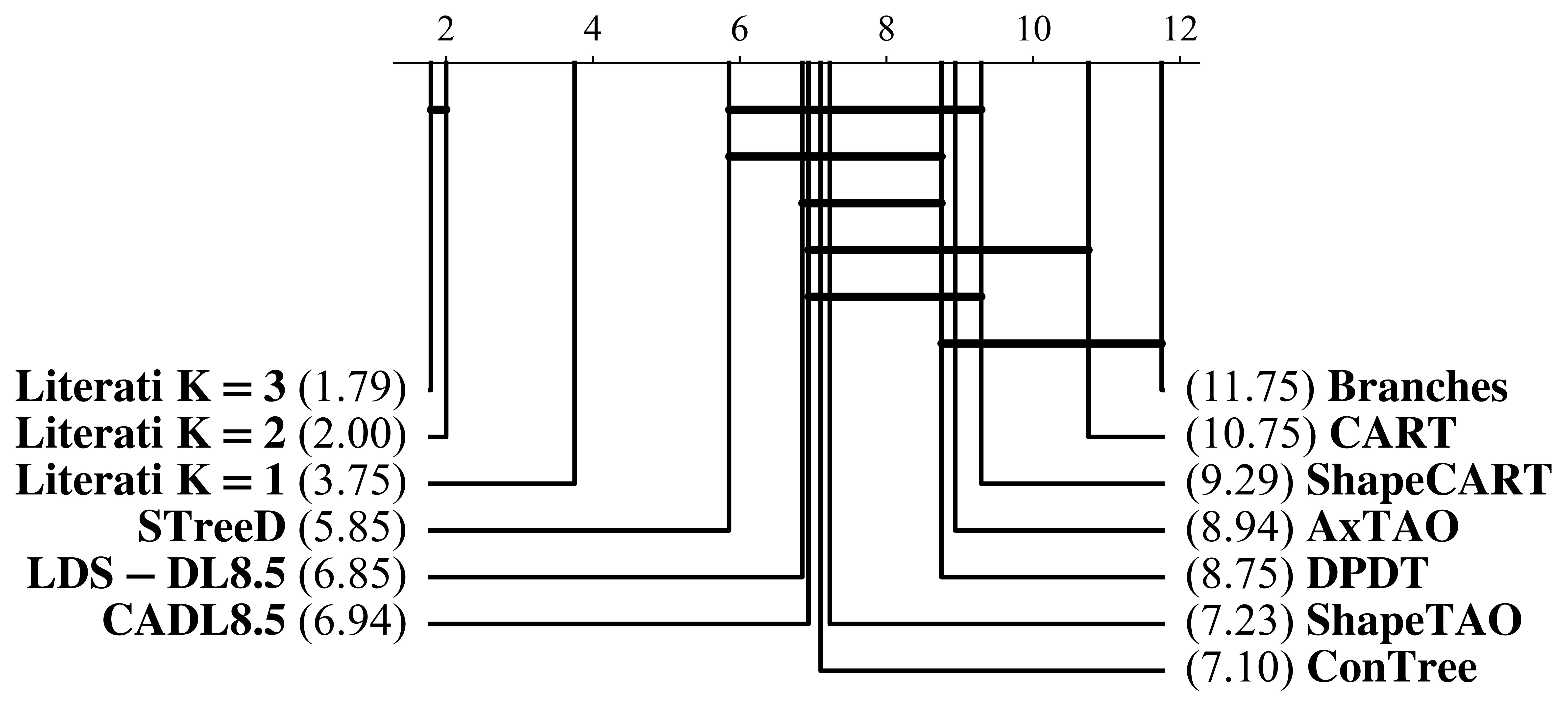}
         \caption{$D=5$}
         \label{fig:d5_trainacc}
     \end{subfigure}
     \hfill
     \begin{subfigure}[b]{0.45\textwidth}
         \centering
         \includegraphics[width=\textwidth]{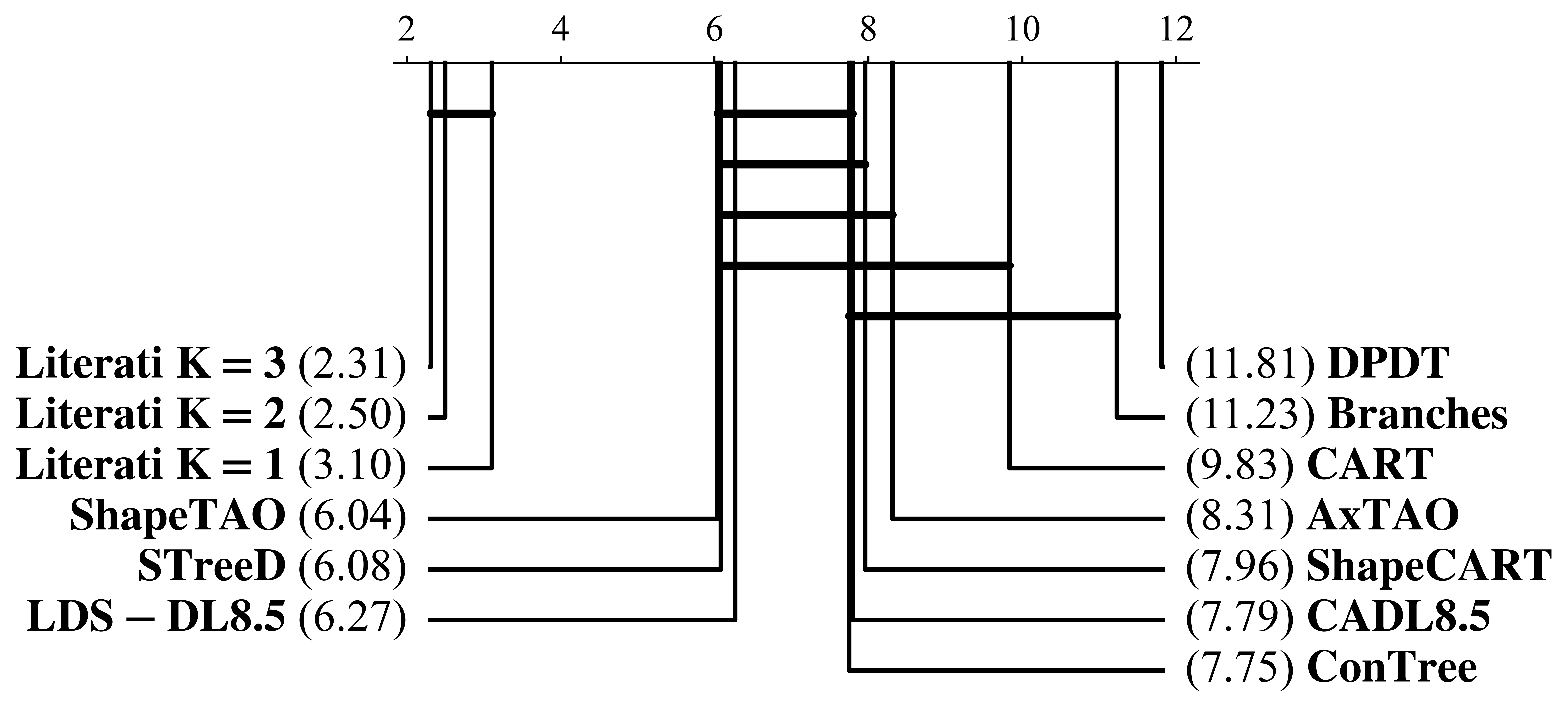}
         \caption{$D=6$}
         \label{fig:d6_trainacc}
     \end{subfigure}
      \caption{Critical Difference Diagram of Train Accuracy for the Optimization capability experiments across datasets and folds}
     \label{fig:opt_cdd_diagrams}
\end{figure}



Figure~\ref{fig:opt_cdd_diagrams} contains critical difference diagrams for each depth on train accuracy. We observe that above $D=3$, Literati across values of $K$ achieves statistically significantly higher train accuracy over all other compared approaches, with Literati $K=2$ and $K=3$ achieving statistically significantly higher train accuracy for $D \in \{4,5\}$ over $K=1$.

\subsection{Light and Depth-Adaptive Configurations}\label{sec:light_config}

{
\setlength{\tabcolsep}{3pt}
\begin{table}[!t]
\centering
\caption{Training accuracy, runtime, and optimality proof rate across depth budgets averaged across all datasets evaluated for DPDT and Literati, organized by Heavy and Light (-L) Discretization. The dashed lines separate Literati from baselines. Bold indicates the best value per depth for training accuracy and runtime.}
\label{tab:light_comparison}
\resizebox{\textwidth}{!}{%
\begin{tabular}{@{}ccccccrrrrcccc@{}}
\toprule
 & \multicolumn{1}{l}{} & \multicolumn{4}{c}{$\uparrow$ Train Acc. (\%)} & \multicolumn{4}{c}{$\downarrow$ Runtime (s)} & \multicolumn{4}{c}{$\uparrow$ Proof Obtained (\%)} \\ \midrule
Model & $K$ & $D=3$ & $D=4$ & $D=5$ & $D=6$ & \multicolumn{1}{c}{$D=3$} & \multicolumn{1}{c}{$D=4$} & \multicolumn{1}{c}{$D=5$} & \multicolumn{1}{c}{$D=6$} & $D=3$ & $D=4$ & $D=5$ & $D=6$ \\ \midrule
\multicolumn{14}{c}{\textbf{Heavy Discretization ($B_d = 16 \ \forall d$)}} \\\noalign{\smallskip}
DPDT & 1 & 83.56 & 87.03 & 77.16 & 71.52 & 46.57 & 488.23 & 2387.41 & 3217.40 & 100.0 & 95.8 & 45.8 & 12.5 \\\cdashline{1-14} \noalign{\smallskip}
\multirow{3}{*}{Literati} & 1 & 83.56 & 87.28 & 89.91 & 92.28 & 18.22 & 259.14 & 1156.29 & 2587.19 & 100.0 & 97.5 & 87.5 & 36.7 \\
 & 2 & 84.03 & 87.75 & 90.36 & 92.58 & 60.82 & 851.76 & 2672.71 & 2817.87 & 100.0 & 87.5 & 33.3 & 22.5 \\
 & 3 & \textbf{84.13} & \textbf{87.90} & \textbf{90.46} & \textbf{92.69} & 162.93 & 1393.71 & 2832.73 & 2782.09 & 100.0 & 82.5 & 25.0 & 24.2 \\ \midrule
\multicolumn{14}{c}{\textbf{Light Discretization ($B_D = 16, B_{D-1} = 8, B_d = 4\ \forall\, 1\leq d<D-1$)}} \\\noalign{\smallskip}
DPDT-L & 1 & 83.48 & 87.11 & 89.64 & 91.5 & 15.09 & 65.21 & 260.43 & 624.28 & 100.0 & 100.0 & 100.0 & 95.8 \\\cdashline{1-14} \noalign{\smallskip}
\multirow{3}{*}{Literati-L} & 1 & 83.48 & 87.11 & 89.64 & 91.83 & \textbf{7.68} & \textbf{33.00} & \textbf{131.70} & \textbf{404.74} & 100.0 & 100.0 & 100.0 & 95.8 \\
 & 2 & 83.85 & 87.46 & 89.99 & 92.26 & 17.00 & 87.67 & 377.58 & 981.87 & 100.0 & 100.0 & 95.8 & 84.2 \\
 & 3 & 83.91 & 87.53 & 90.05 & 92.33 & 24.66 & 135.06 & 564.28 & 1124.52 & 100.0 & 100.0 & 92.5 & 82.5 \\ \bottomrule
\end{tabular}%
}
\end{table}}

\begin{figure}[htbp]
     \centering
     \begin{subfigure}[b]{0.45\textwidth}
         \centering
         \includegraphics[width=\textwidth]{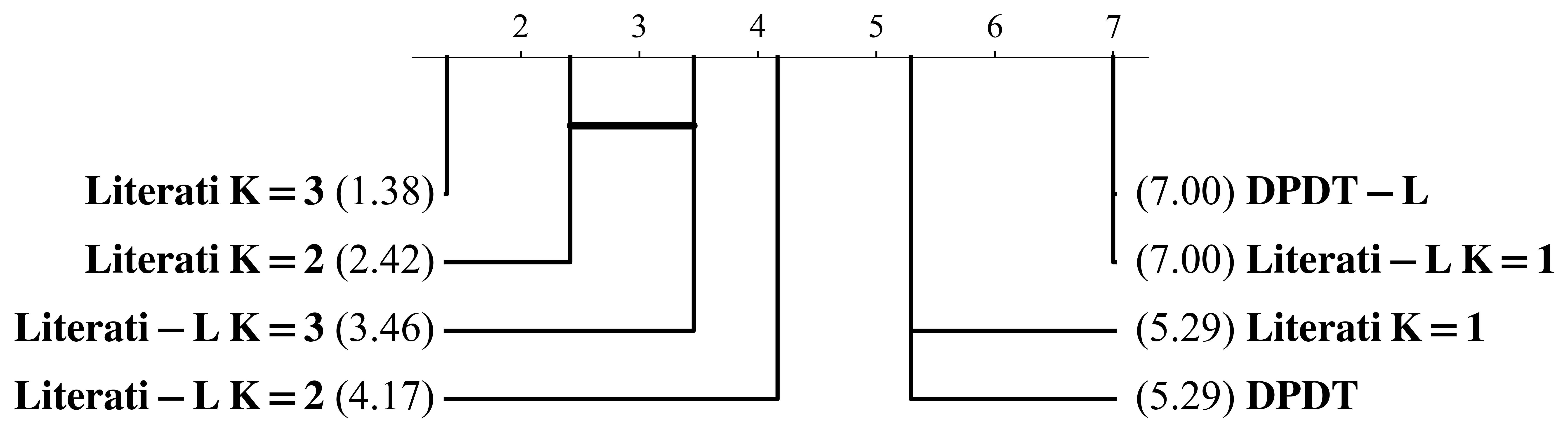}
         \caption{$D=3$}
         \label{fig:d3_trainacclight}
     \end{subfigure}
     \hfill
     \begin{subfigure}[b]{0.45\textwidth}
         \centering
         \includegraphics[width=\textwidth]{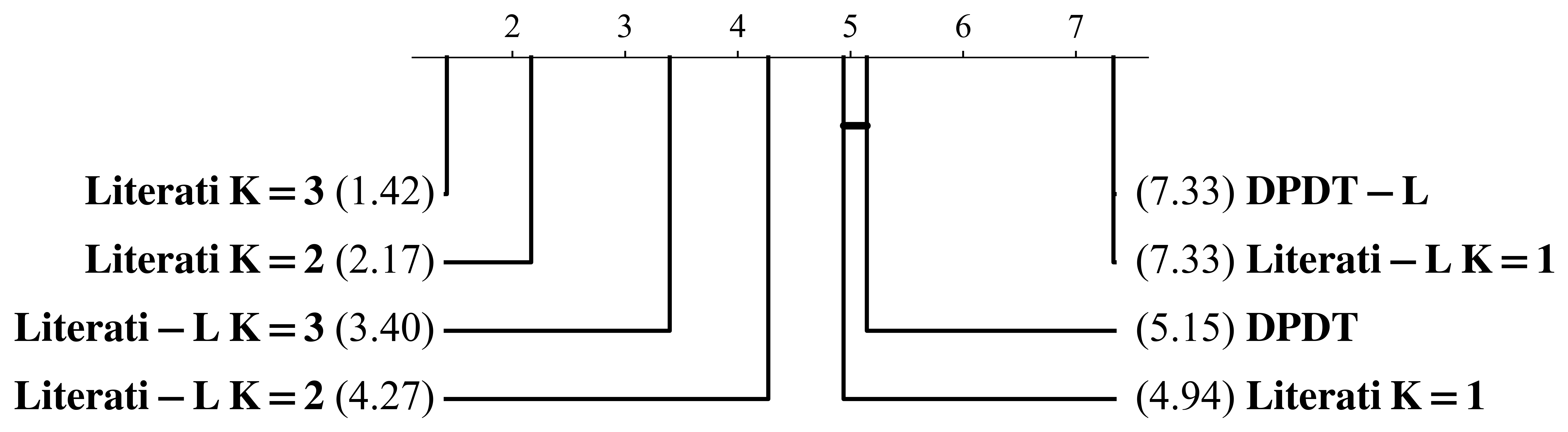}
         \caption{$D=4$}
         \label{fig:d4_trainacclight}
     \end{subfigure}

     \begin{subfigure}[b]{0.45\textwidth}
         \centering
         \includegraphics[width=\textwidth]{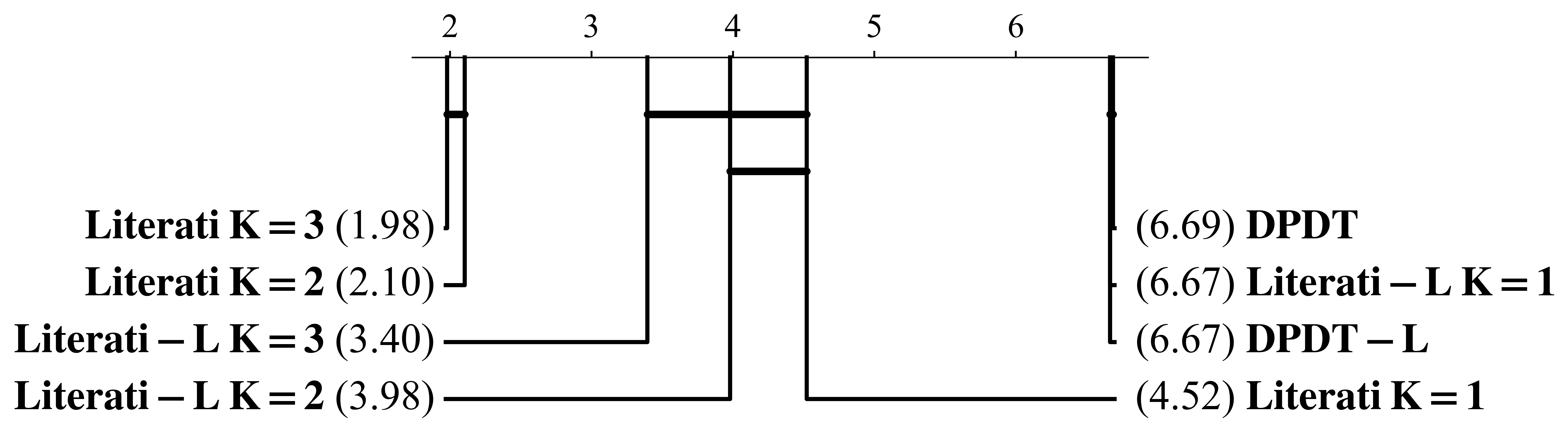}
         \caption{$D=5$}
         \label{fig:d5_trainacclight}
     \end{subfigure}
     \hfill
     \begin{subfigure}[b]{0.45\textwidth}
         \centering
         \includegraphics[width=\textwidth]{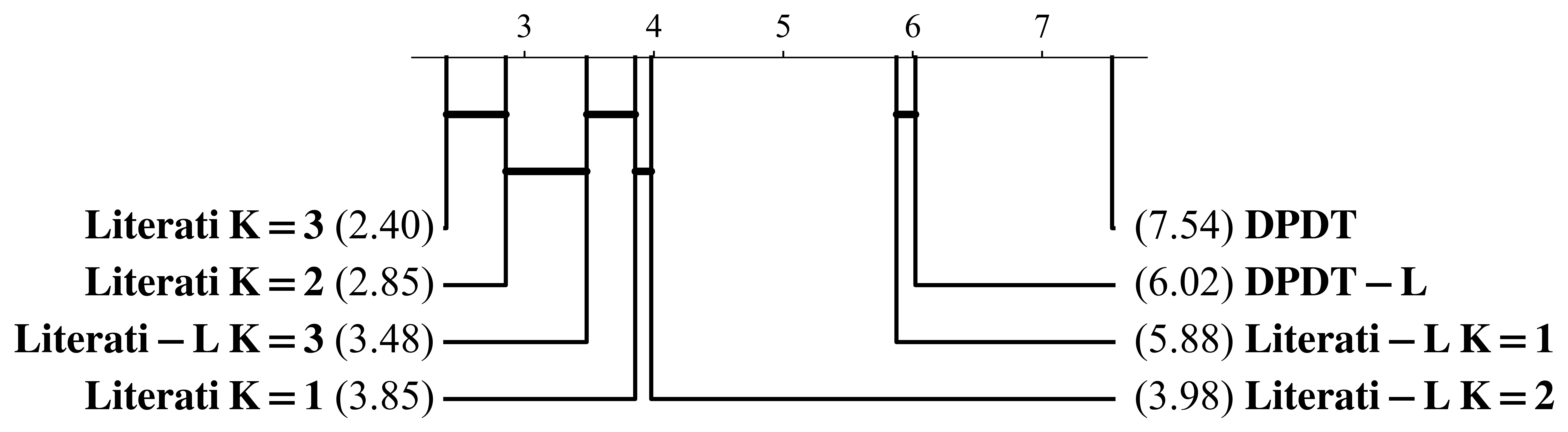}
         \caption{$D=6$}
         \label{fig:d6_trainacclight}
     \end{subfigure}
      \caption{Critical Difference Diagram of Train Accuracy for the Optimization capability experiments across datasets and folds between the heavy and light (-L) variants of DPDT and Literati.}
     \label{fig:opt_cdd_diagramslight}
\end{figure}

In this section, we evaluate a light discretization scheme for DPDT and Literati. More formally, we set $B_D = 16$, $B_{D-1}=8$, and $B_d = 4 \ \forall\, 1\leq d<D-1$ for both DPDT and Literati and compare it to the heavier discretization scheme  ($B_d = 16\ \forall d$) used in Section~\ref{sec:optimization_performance}. Table~\ref{tab:light_comparison} contains training accuracy, runtime, and proof rate for both the heavy and light configurations while Figure~\ref{fig:opt_cdd_diagramslight} contains critical difference diagrams for each depth on training accuracy. Per dataset results can be found in Appendix~\ref{sec:per_dataset_optimality}.

We observe that switching from a heavy to a light discretization configuration reduces runtime by orders of magnitude, though this, as expected, comes at the cost of training accuracy. Furthermore, we observe that we obtain proofs on a majority of datasets across all depths. We observe that under the light configuration, Literati ($K=1$) and DPDT ($K=1$) achieve identical training accuracy across all depths except $D=6$, where neither converges on Mini-Boone. However, because of its anytime nature, Literati can return an intermediate solution, whereas DPDT cannot. When looking at runtime, we observe that the Literati-L $K=1$  achieves lower average runtime than DPDT-L, demonstrating the quality of our search strategy. 

\subsection{Literati Pre-Discretized}\label{sec:prequant}
In addition to adaptive discretization, Literati can operate on pre-discretized candidate sets. In this section, we evaluate Literati using quantile-based pre-discretized features (Literati-PreQuant) and compare its performance with STreeD. We focus on the $K=1$ setting, for which Literati and STreeD operate over equivalent search spaces. To ensure a fair comparison, we also equip Literati-PreQuant with the D2-solver subroutine used by STreeD~\cite{demirovic2022murtree,van2023necessary}. 

{
\setlength{\tabcolsep}{3pt}
\begin{table}[!ht]
\centering
\caption{Comparison of STreeD and Literati under quantile pre-discretization (Literati-PreQuant), averaged over all datasets at each depth budget $D$. Bold marks the best value at each depth between the two pre-discretized methods. Literati with its default adaptive discretization is shown below the solid line for reference.}
\label{tab:prequant_comparison}
\resizebox{\textwidth}{!}{%
\begin{tabular}{@{}ccccccrrrrcccc@{}}
\toprule
 & \multicolumn{1}{l}{} & \multicolumn{4}{c}{$\uparrow$ Train Acc. (\%)} & \multicolumn{4}{c}{$\downarrow$ Runtime (s)} & \multicolumn{4}{c}{$\uparrow$ Proof Obtained (\%)} \\ \midrule
Model & $K$ & $D=3$ & $D=4$ & $D=5$ & $D=6$ & \multicolumn{1}{c}{$D=3$} & \multicolumn{1}{c}{$D=4$} & \multicolumn{1}{c}{$D=5$} & \multicolumn{1}{c}{$D=6$} & $D=3$ & $D=4$ & $D=5$ & $D=6$ \\ \midrule
STreeD & 1 & \textbf{83.50} & 86.57 & 88.44 & 88.26 & 169.90 & 413.52 & 1693.70 & \textbf{2445.12} & 99.2 & \textbf{91.7} & 65.0 & \textbf{37.5} \\
Literati-PreQuant & 1 & \textbf{83.50} & \textbf{87.03} & \textbf{89.50} & \textbf{91.74} & \textbf{58.93} & \textbf{350.05} & \textbf{1623.85} & 2488.06 & \textbf{100.0} & \textbf{91.7} & \textbf{70.8} & 35.8 \\  \midrule
Literati & 1 & 83.56 & 87.28 & 89.91 & 92.28 & 18.22 & 259.14 & 1156.29 & 2587.19 & 100.0 & 97.5 & 87.5 & 36.7 \\ \bottomrule
\end{tabular}%
}
\end{table}}

We observe that Literati-PreQuant consistently achieves higher training accuracy across all depths except for $D=3$ (where it ties STreeD), highlighting the strong anytime performance of our approach. Furthermore, Literati-PreQuant maintains competitive proof rates across all depths, obtaining higher proof rates than STreeD at $D=3$ and $D=5$, whereas STreeD obtains a higher proof rate at $D=6$. 

Comparing Literati-PreQuant with Literati using adaptive discretization, we observe that using adaptive discretization results in higher training accuracy at all depths and lower runtime at all depths except $D=6$.

\subsection{Per Dataset Results}\label{sec:per_dataset_optimality}

{
\fontsize{6}{7}\selectfont
\setlength{\tabcolsep}{3pt}
\begin{longtable}{@{}cccccrrcc@{}}
\caption{Training accuracy, runtime, and optimality proof rate across depth budgets across all datasets evaluated for Depths 3 and 4. The three blocks for each dataset report the normal, light (-L), and Literati-PreQuant configurations, respectively. The dashed lines separate Literati from baselines. Bold indicates the best value per dataset and depth within the normal and light blocks for training accuracy and runtime.
}
\label{tab:full_opt_table_34}\\
\toprule
 &  &  & \multicolumn{2}{c}{$\uparrow$ Train Acc. (\%)} & \multicolumn{2}{c}{$\downarrow$ Runtime (s)} & \multicolumn{2}{c}{$\uparrow$ Proof Obtained   (\%)} \\* \midrule
\endfirsthead
\multicolumn{9}{c}%
{{\bfseries Table \thetable\ continued from previous page}} \\
\toprule
 &  &  & \multicolumn{2}{c}{$\uparrow$ Train Acc. (\%)} & \multicolumn{2}{c}{$\downarrow$ Runtime (s)} & \multicolumn{2}{c}{$\uparrow$ Proof Obtained   (\%)} \\* \midrule
Dataset & Model & K & $D=3$ & $D=4$ & \multicolumn{1}{c}{$D=3$} & \multicolumn{1}{c}{$D=4$} & $D=3$ & $D=4$ \\* \midrule
\endhead
\bottomrule
\endfoot
\endlastfoot
Dataset & Model & K & $D=3$ & $D=4$ & \multicolumn{1}{c}{$D=3$} & \multicolumn{1}{c}{$D=4$} & $D=3$ & $D=4$ \\* \midrule
\multirow{18}{*}{adult} & CART & 1 & $82.32 \pm 0.16$ & $82.81 \pm 0.16$ & $\mathbf{0.09 \pm 0.01}$ & $\mathbf{0.10 \pm 0.02}$ & $0.0$ & $0.0$ \\
 & AxTAO & 1 & $82.34 \pm 0.18$ & $83.31 \pm 0.14$ & $6.24 \pm 2.59$ & $15.88 \pm 4.16$ & $0.0$ & $0.0$ \\
 & ShapeCART & 3 & $82.32 \pm 0.16$ & $82.81 \pm 0.16$ & $2.11 \pm 0.39$ & $2.99 \pm 1.06$ & $0.0$ & $0.0$ \\
 & ShapeTAO & 3 & $82.32 \pm 0.16$ & $83.31 \pm 0.14$ & $8.51 \pm 11.19$ & $28.80 \pm 26.94$ & $0.0$ & $0.0$ \\
 & Branches & 1 & $83.25 \pm 0.13$ & $82.88 \pm 0.33$ & $50.45 \pm 6.22$ & $68.76 \pm 6.36$ & $100.0$ & $0.0$ \\
 & STreeD & 1 & $83.25 \pm 0.13$ & $83.84 \pm 0.12$ & $0.37 \pm 0.05$ & $9.18 \pm 0.57$ & $100.0$ & $100.0$ \\
 & LDS-DL8.5 & 1 & $83.25 \pm 0.13$ & $83.84 \pm 0.12$ & $38.67 \pm 0.53$ & $3600.00 \pm 0.00$ & $100.0$ & $0.0$ \\
\bottomrule
\end{longtable}

}

\newpage
{
\fontsize{6}{7}\selectfont
\setlength{\tabcolsep}{3pt}
\begin{longtable}{@{}cllccrrcc@{}}
\caption{Training accuracy, runtime, and optimality proof rate across depth budgets across all datasets evaluated for Depths 5 and 6. The three blocks for each dataset report the normal, light (-L), and Literati-PreQuant configurations, respectively. The dashed lines separate Literati from baselines. Bold indicates the best value per dataset and depth within the normal and light blocks for training accuracy and runtime.
}
\label{tab:full_opt_table_56}\\
\toprule
  &  &  & \multicolumn{2}{c}{$\uparrow$ Train Acc. (\%)} & \multicolumn{2}{c}{$\downarrow$ Runtime (s $\pm$ std.)} & \multicolumn{2}{c}{$\uparrow$ Proof Obtained (\%)} \\* \midrule
\endfirsthead
\multicolumn{9}{c}%
{{\bfseries Table \thetable\ continued from previous page}} \\
\toprule
  &  &  & \multicolumn{2}{c}{$\uparrow$ Train Acc. (\%)} & \multicolumn{2}{c}{$\downarrow$ Runtime (s $\pm$ std.)} & \multicolumn{2}{c}{$\uparrow$ Proof Obtained (\%)} \\* \midrule
Dataset & Model & K & $D=5$ & $D=6$ & \multicolumn{1}{c}{$D=5$} & \multicolumn{1}{c}{$D=6$} & $D=5$ & $D=6$ \\* \midrule
\endhead
\bottomrule
\endfoot
\endlastfoot
Dataset & Model & K & $D=5$ & $D=6$ & \multicolumn{1}{c}{$D=5$} & \multicolumn{1}{c}{$D=6$} & $D=5$ & $D=6$ \\* \midrule
\multirow{18}{*}{adult} & CART & 1 & $83.28 \pm 0.16$ & $83.72 \pm 0.15$ & $\mathbf{0.10 \pm 0.01}$ & $\mathbf{0.11 \pm 0.01}$ & $0.0$ & $0.0$ \\
 & AxTAO & 1 & $83.75 \pm 0.21$ & $83.98 \pm 0.18$ & $22.57 \pm 4.53$ & $24.69 \pm 5.62$ & $0.0$ & $0.0$ \\
 & ShapeCART & 3 & $83.28 \pm 0.16$ & $83.72 \pm 0.14$ & $3.16 \pm 0.64$ & $4.73 \pm 0.88$ & $0.0$ & $0.0$ \\
 & ShapeTAO & 3 & $83.71 \pm 0.23$ & $83.98 \pm 0.18$ & $58.31 \pm 32.75$ & $79.33 \pm 52.67$ & $0.0$ & $0.0$ \\
 & Branches & 1 & $82.92 \pm 0.32$ & $82.88 \pm 0.45$ & $70.52 \pm 9.93$ & $68.82 \pm 9.11$ & $0.0$ & $0.0$ \\
 & STreeD & 1 & $84.34 \pm 0.10$ & $\mathbf{84.85 \pm 0.08}$ & $337.08 \pm 7.08$ & $3599.66 \pm 0.34$ & $100.0$ & $0.0$ \\
 & LDS-DL8.5 & 1 & $84.34 \pm 0.10$ & $84.60 \pm 0.08$ & $3600.00 \pm 0.00$ & $3600.00 \pm 0.00$ & $0.0$ & $0.0$ \\
\bottomrule
\end{longtable}}

\section{Per Dataset Generalization Results}\label{sec:per_dataset_results}

\begin{table}[!h]
\caption{Test Accuracy (\%) results for our generalization experiments. Highest average test accuracy across folds per dataset is \textbf{bolded}.}
\label{tab:full_gen_table_acc}
\resizebox{\textwidth}{!}{%
\begin{tabular}{@{}llllllllllll@{}}
\toprule
Dataset       & CART                       & AxTAO                      & ShapeCART                  & ShapeTAO                   & Branches                  & STreeD                     & LDS-DL8.5                  & CADL8.5                    & ConTree          & DPDT                       & Literati                   \\ \midrule
adult         & $85.40 \pm 0.20$           & $85.90 \pm 0.30$           & $87.10 \pm 0.30$           & $87.30 \pm 0.30$           & $83.17 \pm 0.29$          & $86.60 \pm 0.30$           & $86.90 \pm 0.40$           & $86.10 \pm 0.20$           & $84.70 \pm 0.30$ & $86.00 \pm 0.30$           & $\mathbf{87.40 \pm 0.20}$  \\
avila         & $68.10 \pm 0.90$           & $69.90 \pm 0.40$           & $71.50 \pm 1.70$           & $75.10 \pm 2.30$           & $63.48 \pm 0.78$          & $73.00 \pm 0.70$           & $74.20 \pm 0.50$           & $72.50 \pm 0.60$           & $63.90 \pm 0.80$ & $73.70 \pm 0.70$           & $\mathbf{79.40 \pm 1.10}$  \\
bank          & $97.90 \pm 0.50$           & $98.80 \pm 0.60$           & $97.50 \pm 1.00$           & $98.40 \pm 0.60$           & $\mathbf{99.27 \pm 0.51}$ & $99.20 \pm 0.70$           & $98.90 \pm 0.40$           & $99.20 \pm 0.50$           & $96.20 \pm 2.50$ & $98.30 \pm 0.30$           & $98.70 \pm 0.30$           \\
bean          & $91.80 \pm 0.10$           & $92.50 \pm 0.20$           & $93.50 \pm 0.20$           & $93.70 \pm 0.30$           & $85.41 \pm 0.56$          & $92.10 \pm 0.40$           & $93.40 \pm 0.20$           & $92.20 \pm 0.40$           & $90.40 \pm 0.70$ & $93.00 \pm 0.20$           & $\mathbf{94.00 \pm 0.40}$  \\
bidding       & $99.50 \pm 0.30$           & $99.50 \pm 0.30$           & $99.70 \pm 0.10$           & $99.70 \pm 0.10$           & $99.49 \pm 0.24$          & $99.50 \pm 0.20$           & $99.70 \pm 0.20$           & $99.70 \pm 0.10$           & $99.00 \pm 0.90$ & $99.70 \pm 0.10$           & $\mathbf{99.80 \pm 0.10}$  \\
\bottomrule
\end{tabular}
}
\end{table}
\begin{table}[!h]
\caption{Runtime (s) of the selected model for our generalization experiments. Lowest average across-folds is \textbf{bolded}.}
\label{tab:full_gen_table_time}
\resizebox{\textwidth}{!}{%
\begin{tabular}{@{}llllllllllll@{}}
\toprule
Dataset       & CART                     & AxTAO              & ShapeCART        & ShapeTAO           & Branches            & STreeD             & LDS-DL8.5         & CADL8.5            & ConTree             & DPDT               & Literati           \\ \midrule
adult         & $\mathbf{0.12 \pm 0.01}$ & $21.68 \pm 6.44$   & $2.37 \pm 1.26$  & $47.51 \pm 5.84$   & $47.96 \pm 4.25$    & $899.68 \pm 0.31$  & $901.72 \pm 0.04$ & $900.15 \pm 0.04$  & $900.09 \pm 0.01$   & $12.73 \pm 1.23$   & $838.64 \pm 58.03$ \\
avila         & $\mathbf{0.06 \pm 0.01}$ & $2.88 \pm 0.29$    & $2.50 \pm 1.73$  & $13.90 \pm 5.06$   & $143.81 \pm 20.48$  & $899.41 \pm 0.30$  & $901.66 \pm 0.03$ & $900.11 \pm 0.04$  & $900.02 \pm 0.03$   & $30.62 \pm 1.29$   & $39.33 \pm 4.47$   \\
bank          & $\mathbf{0.00 \pm 0.00}$ & $0.62 \pm 0.14$    & $0.67 \pm 0.39$  & $15.46 \pm 2.09$   & $55.37 \pm 6.80$    & $0.04 \pm 0.00$    & $0.08 \pm 0.04$   & $0.01 \pm 0.00$    & $0.02 \pm 0.00$     & $0.43 \pm 0.02$    & $0.26 \pm 0.03$    \\
bean          & $\mathbf{0.10 \pm 0.00}$ & $4.53 \pm 1.25$    & $1.97 \pm 0.42$  & $14.47 \pm 3.34$   & $74.16 \pm 2.99$   & $899.55 \pm 0.28$  & $903.94 \pm 0.06$ & $900.06 \pm 0.00$  & $900.01 \pm 0.01$   & $40.25 \pm 0.67$   & $105.04 \pm 7.68$  \\
bidding       & $\mathbf{0.00 \pm 0.00}$ & $0.42 \pm 0.17$    & $0.89 \pm 0.68$  & $17.89 \pm 3.16$   & $48.17 \pm 4.32$   & $9.43 \pm 1.48$    & $903.50 \pm 0.35$ & $142.83 \pm 10.97$ & $103.92 \pm 63.81$  & $1.59 \pm 0.32$    & $12.30 \pm 1.65$   \\
electricity   & $\mathbf{0.08 \pm 0.00}$ & $10.78 \pm 6.30$   & $1.56 \pm 0.30$  & $83.11 \pm 3.26$   & $26.78 \pm 3.97$    & $899.65 \pm 0.32$  & $902.86 \pm 0.29$ & $900.13 \pm 0.01$  & $900.03 \pm 0.02$   & $22.51 \pm 1.16$   & $320.99 \pm 51.27$ \\
\bottomrule
\end{tabular}
}
\end{table}

\newpage

\section{Interpretability }
SGTs inherit many structural qualities of decision trees that make decision trees interpretable~\cite{upadhyaempowering}. As a tree variant, SGTs are modular: like standard decision trees, they decompose into a hierarchical set of rules whose subtrees can be interpreted independently. Because each node's shape function operates on only one feature, it can be directly visualized and queried, which preserves simulability, the ability of a user to trace and reproduce the model's root-to-leaf decision process by hand (See Appendix~\ref{sec:viz} below). Finally, SGTs improve sparsity: the added expressiveness of each node compresses repeated splits into more compact trees, and sparsity is what most directly aids simulability, since a smaller model gives a practitioner a feasible chance of reconstructing its predictions.
\subsection{Visualized Trees}\label{sec:viz}
\begin{figure}[!h]
    \centering
    \includegraphics[width=0.8\linewidth]{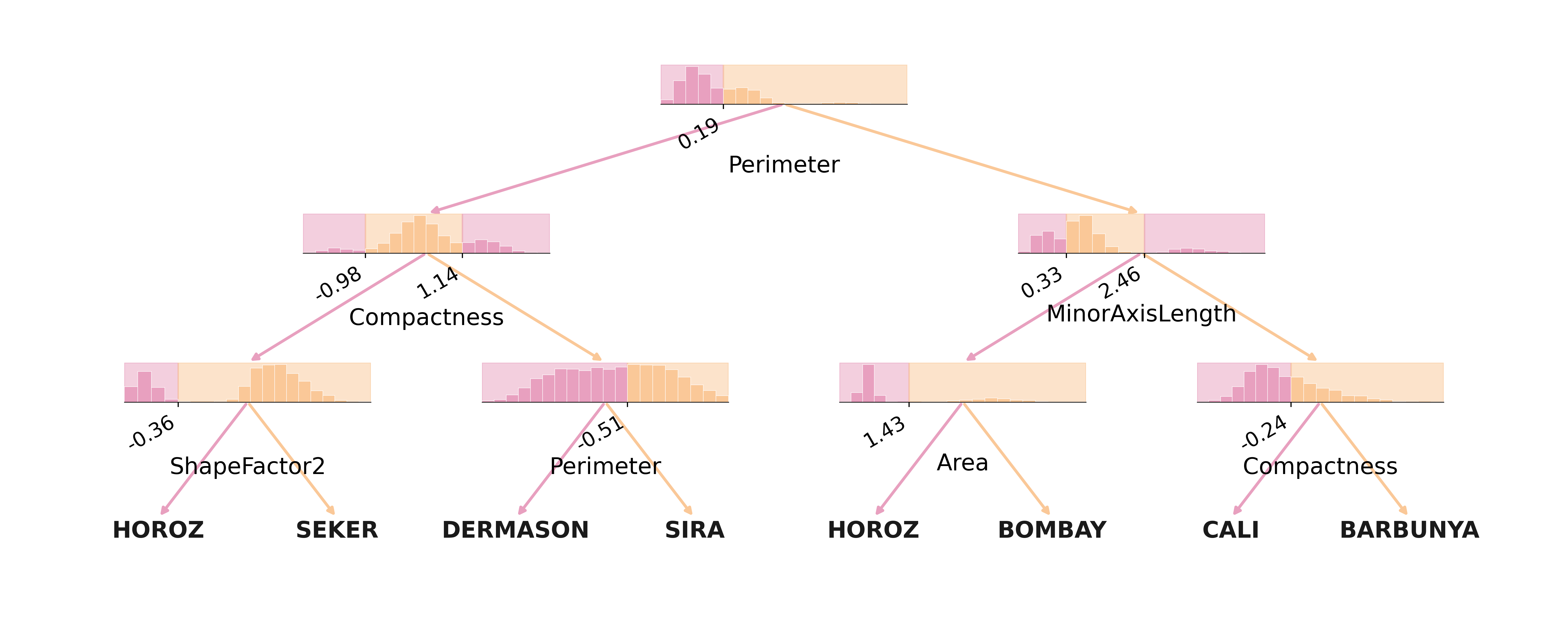}
    \caption{$D=3, K=2$ SGT for Bean}
    \label{fig:bean_tree}
\end{figure}
\begin{figure}[!h]
    \centering
    \includegraphics[width=0.8\linewidth]{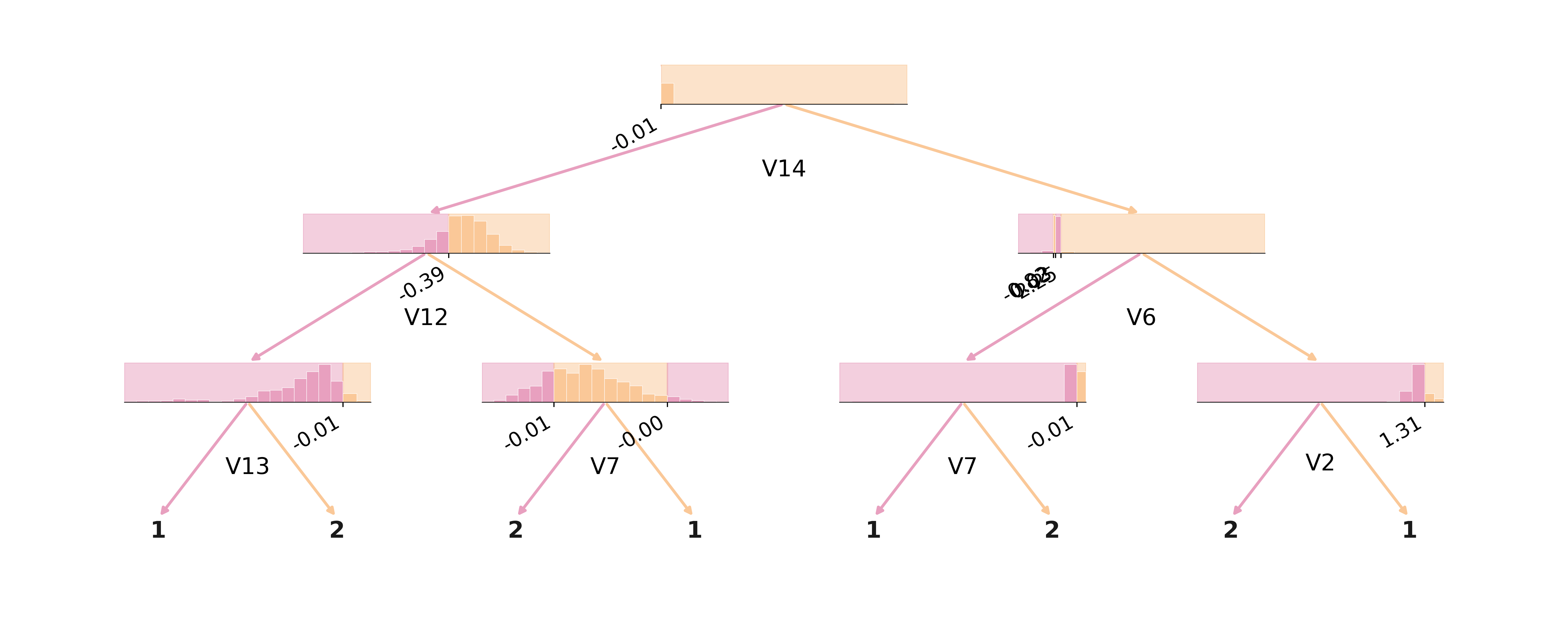}
    \caption{$D=3, K=2$ SGT for Eye-State}
    \label{fig:eyestate_tree}
\end{figure}

\begin{figure}[!h]
    \centering
    \includegraphics[width=0.8\linewidth]{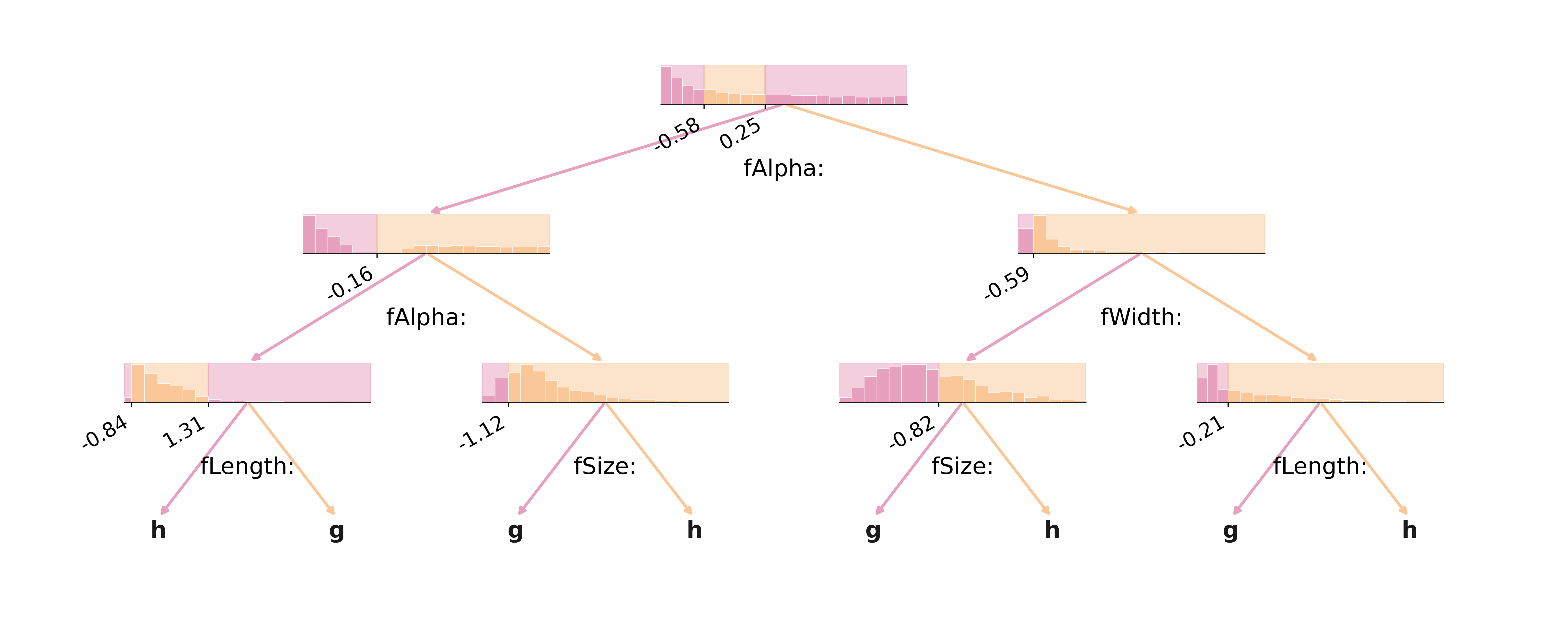}
    \caption{$D=3, K=2$ SGT for Magic}
    \label{fig:magic_tree}
\end{figure}
\clearpage
\section*{NeurIPS Paper Checklist}

\begin{enumerate}

\item {\bf Claims}
    \item[] Question: Do the main claims made in the abstract and introduction accurately reflect the paper's contributions and scope?
    \item[] Answer: \answerYes{} 
    \item[] Justification: We describe all main claims in our abstract and provide a list of contributions in our introduction.

\item {\bf Limitations}
    \item[] Question: Does the paper discuss the limitations of the work performed by the authors?
    \item[] Answer: \answerYes{} 
    \item[] Justification: See Conclusion. 

\item {\bf Theory assumptions and proofs}
    \item[] Question: For each theoretical result, does the paper provide the full set of assumptions and a complete (and correct) proof?
    \item[] Answer: \answerYes{} 
    \item[] Justification: See Appendix for proofs of all theoretical results along with relevant assumptions. 

    \item {\bf Experimental result reproducibility}
    \item[] Question: Does the paper fully disclose all the information needed to reproduce the main experimental results of the paper to the extent that it affects the main claims and/or conclusions of the paper (regardless of whether the code and data are provided or not)?
    \item[] Answer: \answerYes{} 
    \item[] Justification: We provide detailed instructions on our experimental configuration, allowing others to reproduce our experiments if needed. Additionally, we provide detailed pseudocode on how to implement our algorithm.

\item {\bf Open access to data and code}
    \item[] Question: Does the paper provide open access to the data and code, with sufficient instructions to faithfully reproduce the main experimental results, as described in supplemental material?
    \item[] Answer: \answerYes{} 
    \item[] Justification: We utilize open source datasets in our experiments and provide code to replicate our experiments.

\item {\bf Experimental setting/details}
    \item[] Question: Does the paper specify all the training and test details (e.g., data splits, hyperparameters, how they were chosen, type of optimizer) necessary to understand the results?
    \item[] Answer: \answerYes{} 
    \item[] Justification: See Experimental Evaluation section as well as Appendix. 

\item {\bf Experiment statistical significance}
    \item[] Question: Does the paper report error bars suitably and correctly defined or other appropriate information about the statistical significance of the experiments?
    \item[] Answer: \answerYes{} 
    \item[] Justification: We provide statistical significance results along with std.\ deviation across folds.

\item {\bf Experiments compute resources}
    \item[] Question: For each experiment, does the paper provide sufficient information on the computer resources (type of compute workers, memory, time of execution) needed to reproduce the experiments?
    \item[] Answer: \answerYes{} 
    \item[] Justification: See experimental evaluation section as well as appendix. 
    \item[] Guidelines:

\item {\bf Code of ethics}
    \item[] Question: Does the research conducted in the paper conform, in every respect, with the NeurIPS Code of Ethics \url{https://neurips.cc/public/EthicsGuidelines}?
    \item[] Answer: \answerYes{} 
    \item[] Justification: We agree and confirm our work follows the NeurIPS code of ethics.

\item {\bf Broader impacts}
    \item[] Question: Does the paper discuss both potential positive societal impacts and negative societal impacts of the work performed?
    \item[] Answer: \answerYes{} 
    \item[] Justification: See Conclusion. The goal of our work is to provide an improved decision tree algorithm that addresses many of the interpretability concerns with binary axis-aligned trees. As discussed in the conclusion, increased interpretability is important in high-stakes, sensitive fields, and improving interpretability has a positive societal impact. 
    
\item {\bf Safeguards}
    \item[] Question: Does the paper describe safeguards that have been put in place for responsible release of data or models that have a high risk for misuse (e.g., pre-trained language models, image generators, or scraped datasets)?
    \item[] Answer: \answerNA{} 
    \item[] Justification: Our approach has minor risk of misuse and we only consider open source datasets.

\item {\bf Licenses for existing assets}
    \item[] Question: Are the creators or original owners of assets (e.g., code, data, models), used in the paper, properly credited and are the license and terms of use explicitly mentioned and properly respected?
    \item[] Answer: \answerYes{} 
    \item[] Justification: Original producers of relevant code are cited.

\item {\bf New assets}
    \item[] Question: Are new assets introduced in the paper well documented and is the documentation provided alongside the assets?
    \item[] Answer: \answerNA{} 
    \item[] Justification: No assets introduced.

\item {\bf Crowdsourcing and research with human subjects}
    \item[] Question: For crowdsourcing experiments and research with human subjects, does the paper include the full text of instructions given to participants and screenshots, if applicable, as well as details about compensation (if any)? 
    \item[] Answer: \answerNA{} 
    \item[] Justification: No human subjects used.

\item {\bf Institutional review board (IRB) approvals or equivalent for research with human subjects}
    \item[] Question: Does the paper describe potential risks incurred by study participants, whether such risks were disclosed to the subjects, and whether Institutional Review Board (IRB) approvals (or an equivalent approval/review based on the requirements of your country or institution) were obtained?
    \item[] Answer: \answerNA{}
    \item[] Justification: No study participants.

\newpage
\item {\bf Declaration of LLM usage}
    \item[] Question: Does the paper describe the usage of LLMs if it is an important, original, or non-standard component of the core methods in this research? Note that if the LLM is used only for writing, editing, or formatting purposes and does \emph{not} impact the core methodology, scientific rigor, or originality of the research, declaration is not required.
    \item[] Answer: \answerNA{} 
    \item[] Justification: LLMs are not a core part of our work.

\end{enumerate}

\end{document}